\RequirePackage{fix-cm}
\documentclass[smallextended]{svjour3}       
\smartqed  
\usepackage{graphicx}
\usepackage{amsmath,amssymb, natbib, bbm}
\usepackage[linesnumbered,ruled]{algorithm2e}

\newtheorem{assumption}{Assumption}
 
\newcommand{\N}{\mathbb{N}}

\newcommand{\R}{\mathbb{R}}

\newcommand{\E}{\mathbb{E}}

\begin{document}

\title{From Predictions to Prescriptions in Multistage Optimization Problems}


\author{Dimitris Bertsimas         \and
        Christopher McCord 
}


\institute{D. Bertsimas \at
              Sloan School of Management, Massachusetts Institute of Technology, Cambridge, MA 02139 \\
              \email{dbertsim@mit.edu}           
           \and
           C. McCord \at
              Operations Research Center, Massachusetts Institute of Technology, Cambridge, MA 02139 \\
              \email{mccord@mit.edu}
}

\date{Received: date / Accepted: date}

\maketitle

\begin{abstract}
In this paper, we introduce a framework for solving finite-horizon multistage optimization problems under uncertainty in the presence of auxiliary data. We assume the joint distribution of the uncertain quantities is unknown, but noisy observations, along with observations of auxiliary covariates, are available. We utilize effective predictive methods from machine learning (ML), including $k$-nearest neighbors regression ($k$NN), classification and regression trees (CART), and random forests (RF), to develop specific methods that are applicable to a wide variety of problems. We demonstrate that our solution methods are asymptotically optimal under mild conditions. Additionally, we establish finite sample guarantees for the optimality of our method with $k$NN weight functions. Finally, we demonstrate the practicality of our approach with computational examples. We see a significant decrease in cost by taking into account the auxiliary data in the multistage setting.
\end{abstract}

\section{Introduction}
Many fundamental problems in operations research (OR) involve making decisions, dynamically, subject to uncertainty. A decision maker seeks a sequence of actions that minimize the cost of operating a system. Each action is followed by a stochastic event, and future actions are functions of the outcomes of these stochastic events. This type of problem has garnered much attention and has been studied extensively by different communities and under various names (dynamic programming, multistage stochastic optimization, Markov decision process, etc.). Much of this work, dating back to \cite{bellman1952}, has focused on the setting in which the distribution of the uncertain quantities is known a priori.

In practice, it is rare to know the joint distribution of the uncertain quantities. However, in today's data-rich world, we often have historical observations of the uncertain quantities of interest. Some existing methods work with independent and identically distributed (i.i.d.) observations of the uncertainties (cf. \cite{swamy2005}, \cite{shapiro2006}). However, in general, auxiliary data has been ignored in modeling multistage problems, and this can lead to inadequate solutions.

In practice, we often have data, $\{y^1,\ldots,y^N\}$, on uncertain quantities of interest, $Y \in \mathcal{Y} \subset \R^{d_y}$. In addition, we also have data, $\{x^1,\ldots,x^N\}$, on auxiliary covariates, $X \in \mathcal{X} \subset \R^d$, which can be used to predict the uncertainty, $Y$. For example, $Y$ may be the unknown demand for a product in the coming weeks, and $X$ may include data about the characteristics of the particular product and data about the volume of Google searches for the product.

The machine learning (ML) community has developed many methods (cf. \cite{bishop2006}) that enable the prediction of an uncertain quantity ($Y$) given covariates ($X$). These methods have been quite effective in generating predictions of quantities of interest in OR applications (\cite{goel2010}, \cite{gruhl2005}). However, turning good predictions into good decisions can be challenging. One naive approach is to solve the multistage optimization problem of interest as a deterministic problem, using the predicted values of the uncertainties. However, this ignores the uncertainty by using point predictions and can lead to inadequate decisions.

In this paper we combine ideas from the OR and ML communities to develop a data-driven decision-making framework that incorporates auxiliary data into multistage stochastic optimization problems.

\subsection{Multistage Optimization and Sample Average Approximation}
Before proceeding, we first review the formulation of a multistage optimization problem with uncertainty. The problem is characterized by five components:
\begin{itemize}
\item The state at time $t$, $s_t\in S_t$, contains all relevant information about the system at the start of time period $t$.
\item The uncertainty, $y_t\in\mathcal{Y}_t$, is a stochastic quantity that is revealed prior to the decision at time $t$. Throughout this paper, we assume the distribution of the uncertainty at time $t$ does not depend on the current state or past decisions.
\item The decision at time $t$, $z_t \in Z_t(s_t,y_t)\subset\R^{p_t}$, which is chosen after the uncertainty, $y_t$, is revealed.
\item The immediate cost incurred at time $t$, $g_t(z_t)$, which is a function of the decision at time $t$.
\item The dynamics of the system, which are captured by a known transition function that specifies how the state evolves, $s_{t+1} = f_t(z_t)$.
\end{itemize}
We note that it is without loss of generality that the cost and transition functions only depend on the decision variable because the feasible set $Z_t$ is allowed to depend on $s_t$ and $y_t$. To summarize, the system evolves in the following manner: at time $t$, the system is known to be in state $s_t$, when the previously unknown value, $y_t$, is observed. Then the decision $z_t$ is determined, resulting in immediate cost, $g_t(z_t)$, and the system evolves to state $s_{t+1} = f_t(z_t)$.

Consider a finite horizon, $T+1$ stage problem, in which the initial state, $s_0$, is known. We formulate the problem as follows:
\begin{equation}\label{eq:overallnoaux}
\min_{z_0 \in Z_0(s_0)} g_0(z_0) + \E[\tilde{Q}_1(f_0(z_0);Y_1))],
\end{equation}
where
\begin{equation*}
\tilde{Q}_t(s_t;y_t) = \min_{z_t\in Z_t(s_t,y_t)} g_t(z_t) + \E[ \tilde{Q}_{t+1}(f_t(z_t);Y_{t+1})]
\end{equation*}
for $t=1,\ldots,T-1$, and $\tilde{Q}_T(s_T;y_T) = \min\limits_{z_T\in Z_T(s_T,y_T)} g_T(z_T)$. The function $\tilde{Q}_t(s_t;y_t)$ is often called the value function or cost-to-go function. It represents the expected future cost that an optimal policy will incur, starting with a system in state $s_t$ with realized uncertainty $y_t$. Of course, in practice it is impossible to solve this problem because the distributions of $Y_t$ are unknown. All that we know about the distribution of $Y_t$ comes from the available data. 

A popular data-driven method for solving this problem is \emph{sample average approximation} (SAA) \citep{shapiro2014lectures}. In SAA, it is assumed that we have access to independent, identically distributed (i.i.d.) training samples of $Y$, $(y^i_1,\ldots,y^i_T)$ for $i=1,\ldots,N$. The key idea of SAA is to replace the expectations over the unknown distributions of $Y$ with empirical expectations. That is, we replace $\E[\tilde{Q}_T(s_T;y_T)]$ with $\dfrac{1}{N}\sum\limits_{i=1}^N\tilde{Q}_T(s_T;y_T^i)$. With these known, finite distributions of the uncertain quantities, the problem can be solved exactly or approximately by various dynamic programming techniques. Additionally, under certain conditions, the decisions obtained by solving the SAA problem are asymptotically optimal for (\ref{eq:overallnoaux}) \citep{shapiro2003}. The basic SAA method does not incorporate auxiliary data. In practice, this can be accounted for by training a generative, parametric model and applying SAA with samples from this model conditioned on the observed auxiliary data. However, this approach does not necessarily lead to asymptotically optimal decisions, so we instead focus on a variant of SAA that starts directly with the data.

\subsection{Related Work}
Multistage optimization under uncertainty has attracted significant interest from various research communities. \cite{bellman1952} studied these problems under the name dynamic programming. For reference, see \cite{bertsekas2017}. These problems quickly become intractable as the state and action space grow, with a few exceptions that admit closed form solutions, like linear quadratic control \citep{dorato1994}. However, there exists a large body of literature on approximate solution methods (see, e.g., \cite{powell2007}).

When the distribution of the uncertainties is unknown, but data is available, SAA is a common approach \citep{shapiro2006}. Alternative approaches include robust dynamic programming \citep{iyengar2005} and distributionally robust multistage optimization \citep{goh2010}. Another alternative approach is adaptive, or adjustable, robust optimization (cf. \cite{ben2004}, \cite{bertsimas2011}). In this approach, the later stage decisions are typically constrained to be affine or piecewise constant functions of past uncertainties, usually resulting in highly tractable formulations.

In the artificial intelligence community, reinforcement learning (RL) studies a similar problem in which an agent tries to learn an optimal policy by intelligently trying different actions (cf. \cite{sutton1998}). RL methods typically work very well when the exact dynamics of the system are unknown. However, they struggle to incorporate complex constraints that are common in OR problems. A vast literature also exists on bandit problems, which seek to find a series of decisions that balance exploration and exploitation (cf. \cite{berry1985}). Of particular relevance is the contextual bandit problem (cf. \cite{chapelle2011}, \cite{chu2011}), in which the agent has access to auxiliary data on the particular context in which it is operating. These methods have been very effective in online advertising and recommender systems \citep{li2010}.

Recently, the single stage optimization problem with auxiliary data has attracted interest in the OR community. \cite{rudin2014} studied a news-vendor problem in the presence of auxiliary data. \cite{cohen2016} used a contextual bandit approach in a dynamic pricing problem with auxiliary data. \cite{ferreira2015} used data on the sales of past products, along with auxiliary data about the products, to solve a price optimization problem for never before sold products. \cite{bertsimas2014} developed a framework for integrating predictive machine learning methods in a single-stage stochastic optimization problem. Recently, \cite{ban2017} developed a method to solve a multistage dynamic procurement problem with auxiliary data. They used linear or sparse linear regression to build a different scenario tree for each realization of auxiliary covariates. Their approach assumes the uncertainty is a linear function of the auxiliary covariates with additive noise. Our approach is more general because we do not assume a parametric form of the uncertainty.

Statistical decision theory and ML have been more interested in the problems of estimation and prediction than the problem of prescription (cf. \cite{berger2013}). However, of integral importance to our work are several highly effective, yet simple, nonparametric regression methods. These include $k$-nearest neighbor regression \citep{altman1992}, CART \citep{breiman1984}, and random forests \citep{breiman2001}.

\subsection{Contributions and Structure}
In this paper, we consider the analogue of (\ref{eq:overallnoaux}) in the presence of auxilliary data. For each $t=0,\ldots,T-1$, before the decision $z_t$ is made, we observe auxiliary covariates $x_t \in \mathcal{X}_t \subset\R^{d_t}$. Therefore, our training data consists of $(x_0^1,\ldots,x_{T-1}^1),\ldots,(x_0^N,\ldots,x_{T-1}^N)$ and $(y_1^1,\ldots,y_T^1),\ldots,(y_1^N,\ldots,y_T^N)$. By saying training data is i.i.d., we mean observation $i$, $(x_0^i,\ldots,x_{T-1}^i,y_1^i,\ldots,y_T^i)$, was sampled independently of all other observations from the same joint distribution on $\mathcal{X}_1\times\cdots\times\mathcal{X}_{T-1}\times\mathcal{Y}_1\times\cdots\times\mathcal{Y}_T$.

We assume throughout that the auxiliary covariates evolve according to a Markov process, independently of $\{Y_t\}$, i.e., $X_t$ is conditionally independent of $X_0,\ldots,X_{t-2}, Y_1,\ldots,Y_{t-2}$ given $X_{t-1}$. This framework can model more complex dependencies because we are able to choose the auxiliary covariate space. We could for example, append $X_0,\ldots,X_{t-1}$ to the auxiliary covariates observed at time $t$. In addition, we assume $Y_t$ is conditionally independent of all past observations given $X_{t-1}$.

The problem we seek to solve is defined:
\begin{equation}\label{eq:overall}
v^*(x_0) = \min_{z_0 \in Z_0(s_0)} g_0(z_0) + \E[Q_1(f_0(z_0);Y_1,X_1)| X_0=x_0],
\end{equation}
where
\begin{equation*}
Q_t(s_t;y_t,x_t) = \min_{z\in Z_t(s_t,y_t)} g_t(z) + \E[ Q_{t+1}(f_t(z);Y_{t+1},X_{t+1})|X_t=x_t]
\end{equation*}
for $t=1,\ldots,T-2$, with 
\begin{align*}
&Q_{T-1}(s_{T-1};y_{T-1},x_{T-1}) \\
&\;\;\;\;\;= \min_{z\in Z_{T-1}(s_{T-1},y_{T-1})} g_{T-1}(z) + \E[ Q_{T}(f_{T-1}(z);Y_T)|X_{T-1}=x_{T-1}],
\end{align*}
and $Q_T(s_T;y_T) = \min\limits_{z\in Z_T(s_T,y_T)} g_T(z)$. We summarize our key contributions here.
\begin{enumerate}
\item In Section \ref{sec:approach}, we extend the framework introduced by \cite{bertsimas2014} to the multistage setting. Similarly to SAA, we replace the expectations in (\ref{eq:overall}) with sums over the value functions, evaluated at observations of $Y$. However, unlike SAA, we weight the observations according to their relevance to the current problem's auxiliary data, using weight functions inspired by popular machine learning methods.

\item In Section \ref{sec:asymptotics}, we prove the asymptotic optimality and consistency of our method with $k$-nearest neighbor, CART, and random forest weight functions, under fairly mild conditions, for the multistage problem. (We formalize these definitions in that section.) The result for the $k$-nearest neighbor weight function is new for the multistage setting, and the results for the CART and random forest weight functions are new for both the single-stage and multistage settings.

\item In Section \ref{sec:finitesample}, we establish finite sample guarantees for our method with $k$-nearest neighbor weight functions. These guarantees are new for both the single-stage and multistage problems.

\item In Section \ref{sec:computation}, we demonstrate the practical tractability of our method with several computational examples using synthetic data. In addition, our results show that accounting for auxiliary data can have significant value.
\end{enumerate}

\section{Approach}\label{sec:approach}
In this section, we introduce our framework for solving multistage optimization problems under uncertainty in the presence of auxiliary covariates (\ref{eq:overall}). Motivated by the framework developed by \cite{bertsimas2014}, and analogous to SAA, we replace the expectations over an unknown distribution with finite weighted sums of the value functions evaluated at the observations in the data. The weights we use are obtained from ML methods.

First, we use our training data to learn weight functions, $w^t_{N,i}(x_{t-1})$, which quantify the similarity of a new observation, $x_{t-1}$, to each of the training examples, $x_{t-1}^1,\ldots,x_{t-1}^N$. We then replace the conditional expectations in (\ref{eq:overall}) with weighted sums. In particular,
\begin{equation}\label{eq:empirical}
\hat{v}_N(x_0) = \min_{z \in Z_0(s_0)} g_0(z) + \sum_{i=1}^N w^1_{N,i}(x_0)\hat{Q}_1(f_0(z);y^i_1,x^i_1)),
\end{equation}
where
\begin{equation*}
\hat{Q}_t(s_t;y_t,x_t) = \min_{z\in Z_t(s_t,y_t)} g_t(z) + \sum_{i=1}^N w^{t+1}_{N,i}(x_t)\hat{Q}_{t+1}(f_t(z);y^i_{t+1},x^i_{t+1})
\end{equation*}
for $t=1,\ldots,T-2$, with
\begin{align*}
&\hat{Q}_{T-1}(s_{T-1};y_{T-1},x_{T-1}) \\
&\;\;\;\;= \min_{z\in Z_{T-1}(s_{T-1},y_{T-1})} g_{T-1}(z) + \sum_{i=1}^N w^{T}_{N,i}(x_{T-1})\hat{Q}_{T}(f_{T-1}(z);y^i_{T}),
\end{align*}
and $\hat{Q}_T(s_T;y_T) = \min\limits_{z\in Z_T(s_T,y_T)} g_T(z)$.

We note that this is analogous to the sample average approximation method, which can be represented in this framework with the weight functions equal to $\frac{1}{N}$. The weight functions can be computed from various predictive machine learning methods. We list here a few examples that we find effective in practice.

\begin{definition}
Motivated by $k$-nearest neighbor regression \citep{altman1992}, the $k$NN weight function is given by
\begin{equation*}
w_{N,i}(x) = \begin{cases}
1/k, & x^i \text{ is a } k\text{NN of } x, \\
0, & \text{otherwise.}
\end{cases}
\end{equation*}
\end{definition}

\begin{definition}
Motivated by classification and regression trees \citep{breiman1984}, the CART weight function is given by
\begin{equation*}
w_{N,i}(x) = \begin{cases}
1/|R(x)|, & x^i \in R(x), \\
0, & \text{otherwise.}
\end{cases},
\end{equation*}
where $R(x)$ is the set of training points in the same partition as $x$ in the CART model.
\end{definition}

\begin{definition}
Motivated by random forests \citep{breiman2001}, the random forests weight function is given by
\begin{equation*}
w_{N,i}(x) = \dfrac{1}{B}\sum_{b=1}^{B}\dfrac{1}{|R^b(x)|}\mathbbm{1}\{x^i \in R^b(x)\},
\end{equation*}
where $R^b(x)$ is the set of training points in the same partition as $x$ in tree $b$ of the random forest.
\end{definition}
We offer two observations on the formulation in (\ref{eq:empirical}) before proceeding.
\begin{enumerate}
\item In SAA, we are justified in replacing $\E[\theta]$ by $\frac{1}{N}\sum\limits_{i=1}^N \theta^i$ by the strong law of large numbers. We will see in Section \ref{sec:asymptotics} that a conditional strong law of large numbers holds under certain conditions, and this justifies replacing $\E[h(Y_t)\mid X=x]$ with $\sum\limits_{i=1}^N w_{N,i}(x)h(y_t^i)$.
\item If the weight functions are nonnegative and sum to one (as is the case with those presented here), then we can think of this formulation as defining a dynamic programming problem in which the uncertain quantities have a known, finite distribution. This means we can readily apply exact and approximate dynamic programming algorithms to solve (\ref{eq:empirical}). For the reader's convenience, we provide a decomposition algorithm, tailored to our approach, in Appendix \ref{app:decomp}, which can be used to exactly solve small to moderately sized problems.
\end{enumerate}

\subsection{Notation}
We summarize the relevant notation we use in Table \ref{tbl:notation}. We use lower case letters for $x_t^i$ and $y_t^i$ to denote observed quantities in the data and capital letters, $X_t^i$ and $Y_t^i$, to denote random quantities. When we discuss the asymptotic optimality of solutions in Definitions \ref{def:strong} and \ref{def:weak}, the data are random quantities, and thus solutions to (\ref{eq:empirical}) are also random. For notational convenience, we sometimes write $Q_T(s_T;y_T,x_T)$ even though $Q_T$ does not depend on $x_T$ (because the auxiliary data observed after the last decision is made is irrelevant to the problem).

\begin{table}
\centering
\begin{tabular}{|l|l|}
\hline
$x_t$ & Auxiliary data observed at time $t$, $x_t \in \mathcal{X}_t \subset \R^{d_t}$ \\
$s_t$ & State of the system at the beginning of time period $t$, $s_t \in S_t$ \\
$z_t$ & Decision made at time $t$, $z_t\in Z_t(s_t,y_t)\subset\R^{p_t}$ \\
$y_t$ & Uncertain quantity observed after the decision at time $t-1$, $y_t\in\mathcal{Y}_t$ \\
$f_t(z_t)$ & Transition function that gives the evolution of the state to $s_{t+1}$ \\
$g_t(z_t)$ & Cost of decision $z_t$ at time $t$ \\
 \hline
 $w_{N,i}^t(x_{t-1})$ & Weight function for stage $t$, gives weighting for training sample $i$ \\
 $Q_t(s_t;y_t,x_t)$ & Value function in full information problem (\ref{eq:overall}) \\
 $\hat{Q}_t(s_t;y_t,x_t)$ & Value function in approximate problem (\ref{eq:empirical}) \\
 \hline
 $v^*(x_0)$ & Optimal objective value of full information problem (\ref{eq:overall}) \\
 $\hat{v}_N(x_0)$ & Optimal objective value of approximate problem (\ref{eq:empirical}) \\
 \hline
\end{tabular}
\caption{Summary of notation.}\label{tbl:notation}
\end{table}

\section{Asymptotic Optimality}\label{sec:asymptotics}
In the setting without auxiliary covariates, under certain conditions, the SAA estimator is strongly asymptotically optimal \citep{shapiro2003}. Here, we provide similar results for the multistage setting with auxiliary data.

\begin{definition}\label{def:strong}
We say $\hat{z}^N_0(x_0)$, a sequence of optimal solutions to (\ref{eq:empirical}), is \emph{strongly asymptotically optimal} if, for $x_0$ almost everywhere (a.e.),
\begin{enumerate}
\item The estimated cost of $\hat{z}^N_0(x_0)$ converges to the true optimal cost:
\begin{equation}\label{eq:costcons}
\min_{z\in Z_0(s_0)} g_0(z) + \sum_{i=1}^N w_{N,i}^1(x_0)\hat{Q}_1(f_0(z);Y_1^i,X_1^i) \to v^*(x),
\end{equation}
almost surely.
\item The true cost of $\hat{z}^N_0(x_0)$ converges to the true optimal cost:
\begin{equation*}
\E[g_0(\hat{z}_0^N) + Q_1(f_0(\hat{z}_0^N);Y_1,X_1)| X_0=x_0,\hat{z}_0^N] \to v^*(x_0),
\end{equation*}
almost surely.
\item The limit points of $\{\hat{z}_0^N(x_0)$ are contained in the set of true optimal solutions:
\begin{equation*}
L(\{\hat{z}_0^N(x_0) : N\in \mathbb{N}\}) \subset \arg\min_{z\in Z_0(s_0)} \E[g_0(z) + Q_0(f_0(z);Y_1,X_1)| X_0=x_0],
\end{equation*}
almost surely, where $L(S)$ denotes the limit points of the set $S$.
\end{enumerate}
\end{definition}

\begin{definition}\label{def:weak}
We say $\hat{z}^N_0(x_0)$, a sequence of optimal solutions to (\ref{eq:empirical}), is \emph{weakly asymptotically optimal} if, for $x_0$ almost everywhere (a.e.),
\begin{enumerate}
\item The estimated cost of $\hat{z}^N_0(x_0)$ converges to the true optimal cost in probability:
\begin{equation}
\min_{z\in Z_0(s_0)} g_0(z) + \sum_{i=1}^N w_{N,i}^1(x_0)\hat{Q}_1(f_0(z);Y_1^i,X_1^i) \to_P v^*(x).
\end{equation}
\item The true cost of $\hat{z}^N_0(x_0)$ converges to the true optimal cost in probability:
\begin{equation*}
\E[g_0(\hat{z}_0^N) + Q_1(f_0(\hat{z}_0^N);Y_1,X_1)| X_0=x_0,\hat{z}_0^N] \to_P v^*(x_0).
\end{equation*}
\end{enumerate}
\end{definition}

\subsection{$k$-Nearest Neighbor Weight Functions}
We begin by defining some assumptions under which asymptotic optimality will hold for (\ref{eq:empirical}) with the $k$-nearest neighbor weight functions.

\begin{assumption}[Regularity]\label{as:regularity}
For each $t=0,\ldots,T$, there exists a closed and bounded set $W_t$ such that for any $s_t$ and $y_t$, the feasible region, $Z_t(s_t,y_t)$, is contained in $W_t$.
\end{assumption}

\begin{assumption}[Existence]\label{as:existence}
The full information problem (\ref{eq:overall}) is well defined for all stages, $t=0,\dots,T-1$: $\E|g_t(z_t) + Q_{t+1}(f_t(z_t);Y_{t+1},X_{t+1})| < \infty$ for all $z_t\in W_t$ and $Z_t(s_t,y_t)$ is nonempty for all $s_t$ and $y_t$.
\end{assumption}

\begin{assumption}[Continuity]\label{as:continuity}
The function $g_{t-1}(z_{t-1})+Q_t(f_{t-1}(z_{t-1});y_t,x_t)$ is equicontinuous in $z_{t-1}$ at all stages. That is, for $t=1,\ldots,T$, $\forall z\in W_{t-1},\epsilon > 0$, $\exists \delta > 0$ such that
\begin{equation*}
\sup\limits_{y_t\in \mathcal{Y}_t}\sup\limits_{x_t\in\mathcal{X}_t} |g_{t-1}(z)+Q_t(f_{t-1}(z);y_t,x_t) -g_{t-1}(z')- Q_t(f_{t-1}(z');y_t,x_t)| \le \epsilon
\end{equation*}
for all $z' \in \{z' : ||z-z'|| \le \delta\}\cap W_t$. Additionally, the final cost function, $g_T(z_T)$ is continuous.
\end{assumption}

%
%

\begin{assumption}[Distribution of Uncertainties]\label{as:uniform}
The following hold:
\begin{enumerate}
\item
The stochastic process $X_0,X_1,\ldots,X_{T-1}$ satisfies the Markov property.
\item
For each $t = 1,\ldots,T$, $Y_t$ is conditionally independent of $Y_1,\ldots,Y_{t-1}$ and $X_0,\ldots,X_{t-2}$ given $X_{t-1}$.
\item
For each $t=0,\ldots,T-1$, the support of $X_t$, $\mathcal{X}_t$, is compact.
\item
The noise terms have uniformly bounded tails, i.e., defining $N_t(s_t;x_{t-1}) = Q_t(s_t;Y_t,X_t) - \E[Q_t(s_t;Y_t,X_t)|X_{t-1}=x_{t-1}]$, there exists $\lambda > 0$ such that
\begin{equation*}
\max_{t=1,\ldots,T} \sup_{x_{t-1}\in\mathcal{X}_{t-1}}\sup_{s_t\in S_t} \E[e^{\lambda|N_t(s_t;x_{t-1})|}|X_{t-1}=x_{t-1}] < \infty.
\end{equation*}
\item For all $t$ and $s_t$, $\E[Q_t(s_t;Y_t,X_t)|X_{t-1}=x]$ is a continuous function of $x$.
\end{enumerate}
\end{assumption}

We remark that Assumption \ref{as:continuity} does not preclude the possibility of integral constraints on $z_t$. In fact, if $Z_t$ is a finite discrete set, the assumption is automatically satisfied (choose $\delta < \min_{z,z'\in Z_t} ||z-z'||$). Condition 4 of Assumption \ref{as:uniform} can be satisfied if $N_t(s_t;x_{t-1})$ are uniformly bounded random variables, are subgaussian random variables with uniformly bounded subgaussian norms, or are subexponential random variables with uniformly bounded subexponential norms. With these assumptions, we have the following result regarding the asymptotic optimality of (\ref{eq:empirical}) with the $k$-nearest neighbor weight functions.

\begin{theorem}\label{thm:knn}
Suppose Assumptions \ref{as:regularity}-\ref{as:uniform} hold, and the training data is i.i.d. Let $w^t_{N,i}(x_{t-1})$ be the $k$-nearest neighbor weight functions for $t=1,\ldots,T$ with $k = \min\{\lceil CN^\delta\rceil,N-1\}$ for $C > 0, \delta \in (0,1)$. Then $\{\hat{z}_0^N(x)\}$, a sequence of optimal solutions to (\ref{eq:empirical}), is strongly asymptotically optimal.
\end{theorem}

Comparing the result with that of \citet{bertsimas2014} for the single-stage problem, we see it is quite similar. Both theorems assume regularity, existence, and continuity. The difference in the multistage setting is that these assumptions must hold for the value function at each stage. We also note that \citet{bertsimas2014} listed several sets of regularity assumptions that could hold, whereas we only list one for clarity. However, the extension of the other sets of regularity assumptions to the multistage setting is straightforward.

To prove this result, we rely on several technical lemmas. Lemmas \ref{lemma:uniform},\ref{lemma:opt}, and \ref{lemma:final} are refined versions of results originally stated in \cite{bertsimas2014}.
\begin{lemma}\label{lemma:uniform}
Suppose for each $N\in\N$, $\hat{C}_N(z|x)$ and $C(z|x)$ are equicontinuous functions, i.e., $\forall\epsilon > 0$  and $z\in\mathcal{Z}$,  $\exists\delta > 0$ s.t. $\sup\limits_{z'\in B_\delta(z)\cap\mathcal{Z}}\sup\limits_{x\in\mathcal{X}}|\hat{C}_N(z|x) - \hat{C}_N(z'|x)| \le \epsilon$, and likewise for $C$. If for every $z\in\mathcal{Z}$,
\begin{equation*}
\sup_{x\in\mathcal{X}}\left|\hat{C}_N(z|x) - C(z|x) \right| \to 0,
\end{equation*}
then the convergence is uniform over any compact subset of $\mathcal{Z}$.
\end{lemma}
\begin{proof}
Let $z_N \in \mathcal{Z}$ be a sequence that converges to $z\in \mathcal{Z}$. For any $\epsilon > 0$, by assumption, $\exists \delta > 0$ such that $\sup\limits_{z'\in B_\delta(z)\cap\mathcal{Z}}|\hat{C}_N(z|x) - C(z'|x)| \le \epsilon/2$ for all $x$. By the convergence of $z_N$, there exists an $N_1$ such that $\forall N\ge N_1$, $z_N\in B_\delta(z)\cap\mathcal{Z}$. This implies, $\forall N \ge N_1$, $|\hat{C}_N(z_N|x) - \hat{C}_N(z|x)| \le \epsilon/2$ for all $x$. Additionally, by assumption, $\exists N_2$ s.t. $\forall N\ge N_2$, $\sup\limits_{x\in\mathcal{X}}\left|\hat{C}_N(z|x) - C(z|x)\right| \le \epsilon / 2$. Therefore, $\forall N \ge \max(N_1,N_2)$
\begin{equation*}
\sup\limits_{x\in\mathcal{X}}\left|\hat{C}_N(z_N|x) - C(z|x)\right| \le \epsilon. 
\end{equation*}
Therefore, for any convergent sequence in $\mathcal{Z}$, $z_N\to z$, $\sup\limits_{x\in\mathcal{X}} |\hat{C}_N(z_N|x) - C(z|x)|\to0$.

Given a compact subset $E \subset \mathcal{Z}$, suppose that $\sup\limits_{z\in E}\sup\limits_{x\in\mathcal{X}}\hat{C}_N(z|x) - C(z|x)| \not\to 0$. This implies $\exists \epsilon > 0$ and a sequence $z_N\in E$ such that $\sup\limits_{x\in\mathcal{X}}|\hat{C}_N(z_N|x) - C(z_N|x)| > \epsilon$ occurs infinitely often. Define $z_{N_k}$ to be the subsequence for which this event occurs. Since $E$ is compact, by the Bolzano-Weirestrass theorem, $z_{N_k}$ has a convergent subsequence in $E$. If we define $\{s_N\}$ to be this subsubsequence, we have that $s_N\to s\in E$ and $\sup\limits_{x\in\mathcal{X}}|\hat{C}_N(s_N|x) - C(s_N|x)| > \epsilon$ for all $N$. We have $\sup\limits_{x\in\mathcal{X}}|\hat{C}_N(s_N|x) - C(s_N|x)| \le \sup\limits_{x\in\mathcal{X}}|\hat{C}_N(s_N|x) - C(s|x)| + \sup\limits_{x\in\mathcal{X}}|C(s|x) - C(s_N|x)|$. The first term converges to 0 because of what we showed above. The second term converges to 0 by the equicontinuity assumption. This is a contradiction, so it must be that $\sup\limits_{z\in E}\sup\limits_{x\in\mathcal{X}}|\hat{C}(z|x) - C(z|x)| \to 0$ for any compact set $E\subset\mathcal{Z}$.
\end{proof}
This lemma shows that, given an equicontinuity assumption, pointwise convergence of a function implies uniform convergence over a compact set. We will apply this with $C=Q_t$. The following two lemmas establish that strong asymptotic optimality follows from the uniform convergence of the objective of (\ref{eq:empirical}) to that of (\ref{eq:overall}).

\begin{lemma}\label{lemma:uniformization}
Suppose $\mathcal{Z}$ is a compact set and $h(z)$ and $g(z)$ are two continuous functions. If $z^* \in \arg\min\limits_{z\in\mathcal{Z}} h(z)$ and $z' \in \arg\min\limits_{z\in\mathcal{Z}} g(z)$, then
\begin{equation*}
|h(z^*) - g(z')| \le \sup_{z\in\mathcal{Z}} |h(z) - g(z)|,
\end{equation*}
and
\begin{equation*}
|h(z^*) - h(z')| \le 2\sup_{z\in\mathcal{Z}} |h(z) - g(z)|.
\end{equation*}
\end{lemma}
\begin{proof}
First, because of the optimality of $z^*$ and $z'$,
\begin{align*}
&|h(z^*)-g(z')| \\
&\le\;\;\;\; \begin{cases}
|h(z^*) - g(z^*)|, & h(z^*) \le g(z') \\
|h(z') - g(z')|, & h(z^*) > g(z')
\end{cases} \\
&\le \sup_{z\in\mathcal{Z}} |h(z)-g(z)|.
\end{align*}
Second,
\begin{align*}
|h(z^*) - h(z')| &\le |h(z') - g(z')| + |g(z') - h(z^*)|\\
&\le 2\sup_{z\in\mathcal{Z}} |h(z)-g(z)|.
\end{align*}
\end{proof}

\begin{lemma}\label{lemma:opt}
Fix $x\in\mathcal{X}$ and suppose $\sup\limits_{z\in\mathcal{Z}} |\hat{C}_N(z|x) - C(z|x)| \to 0$ as $N\to\infty$ and $C(z|x)$ is a continuous function of $z$. In addition, suppose constraint set $\mathcal{Z}$ is nonempty, closed, and bounded. Any sequence $z_N\in\arg\min\limits_{z\in\mathcal{Z}} \hat{C}_N(z|x)$ for $N\in\N$ has all of its limit points contained in $\arg\min\limits_{z\in\mathcal{Z}} C(z|x)$.
\end{lemma}
\begin{proof}
Suppose there is a subsequence $z_{N_k}$, converging to $z \notin \arg\min_{z\in\mathcal{Z}} C(z|x)$. (We must still have that $z\in\mathcal{Z}$ because $\mathcal{Z}$ is compact.) Let $\epsilon = C(z|x) - \min\limits_{z\in\mathcal{Z}} C(z|x) > 0$. By the continuity assumption, $\exists k_1$ such that for all $k\ge k_1$ $|C(z_{N_k}|x) - C(z|x)| \le \epsilon / 4$. Additionally, by assumption, we can find a $k_2$ such that $\forall k \ge k_2$, $|\hat{C}_{N_k}(z_{N_k}|x) - C(z_{N_k}|x)| \le \sup\limits_{z\in\mathcal{Z}} |\hat{C}_{N_k}(z|x) - C(z|x)| \le \epsilon / 4$ This implies that for any $k \ge \max(k_1,k_2)$,
\begin{equation*}
\min_{z\in\mathcal{Z}} \hat{C}_{N_k}(z|x) = \hat{C}_{N_k}(z_{N_k}|x) \ge C(z_{N_k}|x) - \epsilon / 4 \ge C(z|x) - \epsilon / 2 = \min_{z\in\mathcal{Z}} C(z|x) + \epsilon/2.
\end{equation*}
From lemma \ref{lemma:uniformization}, we know that
\begin{equation*}
|\min_{z\in\mathcal{Z}} \hat{C}_{N_k}(z|x) -  \min_{z\in\mathcal{Z}} C(z|x)| \le \sup_{z\in\mathcal{Z}} |\hat{C}_{N_k}(z|x) - C(z|x)|,
\end{equation*}
which goes to 0 as $k \to \infty$, and is thus a contradiction. Therefore, all limit points of $\arg\min_z \hat{C}_N(z|x)$ must be contained in $\arg\min_z C(z|x)$.
\end{proof}

Lemma \ref{lemma:final} shows that, given an equicontinuity assumption, if a function indexed by $z$ converges almost surely for each $z$ in a compact set, then the convergence holds almost surely for all $z$.

\begin{lemma}\label{lemma:final}
Suppose for each $N\in\N$, $\hat{C}_N(z|x)$ and $C(z|x)$ are equicontinuous functions: i.e., $\forall\epsilon > 0$  and $z\in\mathcal{Z}$,  $\exists\delta > 0$ s.t. $\sup\limits_{z'\in B_\delta(z)\cap\mathcal{Z}}\sup\limits_{x\in\mathcal{X}}|\hat{C}_N(z|x) - \hat{C}_N(z'|x)| \le \epsilon$, and likewise for $C$. In addition, suppose that for each $z \in \mathcal{Z}$, $\sup\limits_{x\in\mathcal{X}}\left|\hat{C}_N(z|x) - C(z|x)\right| \to 0$ almost surely ($\hat{C}_N(z|x)$ is a random quantity). Furthermore, assume $\mathcal{Z}$ is compact. Then, almost surely, $\sup\limits_{x\in\mathcal{X}}\left|\hat{C}_N(z|x) - C(z|x)\right|$ for all $z\in\mathcal{Z}$.
\end{lemma}
\begin{proof}
Let $\mathcal{Z}' = \mathcal{Z}\cap \mathbb{Q}^d \cup \{\text{isolated points of } \mathcal{Z}\}$. Because $\mathcal{Z}'$ is countable, for $x$ almost everywhere, $P\left(\bigcap_{z'\in\mathcal{Z}'}\{\sup_{x\in\mathcal{X}}|\hat{C}_N(z'|x) - C(z'|x)|\to0\}\right) = 1$ by the continuity of probability measures. For any sample path for which this occurs, consider any $z \in \mathcal{Z}$. We have, for any $z'\in\mathcal{Z}'$,
\begin{align*}
\sup_{x\in\mathcal{X}}|\hat{C}_N(z|x) - C(z|x)| &\le \sup_{x\in\mathcal{X}}|\hat{C}_N(z|x) - \hat{C}_N(z'|x)| + \sup_{x\in\mathcal{X}}|C(z'|x) - C(z|x)| \\
&\;\;\;\;+ \sup_{x\in\mathcal{X}}|\hat{C}_N(z'|x) - C(z'|x)|.
\end{align*}
By equicontinuity and the density of $\mathcal{Z}'$, we can pick $z'\in\mathcal{Z}'$ such that each of the first two terms is less $\epsilon / 3$ for any $\epsilon > 0$. By assumption, we can also find an $N_1$ such that the third term is bounded by $\epsilon/3$ for all $N\ge N_1$, so we have $\sup\limits_{x\in\mathcal{X}}|\hat{C}_N(z|x) - C(z|x)| \le \epsilon$ for all $N\ge N_1$. This is true for any $z \in \mathcal{Z}$ for this particular sample path. Since the set of sample paths for which this is true constitutes a measure 1 event, we have the desired result.
\end{proof}

We also restate a result from \citet{biau2015} (Theorem 12.1) regarding the uniform consistency of the $k$-nearest neighbor regression estimator.
\begin{lemma}\label{lemma:knnconsistency}
Let $(X^1,Y^1),\ldots,(X^N,Y^N) \in \R^{d}\times\R$ be i.i.d. observations of random variables $(X,Y)$. Assume $X$ has support on a compact set $\mathcal{X}\subset \R^d$ and there exists $\lambda > 0$ such that
\begin{equation*}
\sup_{x\in\R^d} \E[\exp(\lambda |Y - \E[Y|X=x]|)|X=x] < \infty.
\end{equation*}
In addition, assume $\E[Y|X=x]$ is a continuous function. For some $C>0,\delta\in(0,1)$, let $k_N = \min\{\lceil CN^\delta\rceil,N-1\}$. If $m_N(x)$ is the $k_N$ nearest neighbor regression estimator for $Y$ and $m(x) = \E[Y|X=x]$, then
\begin{equation*}
\sup_{x\in\mathcal{X}} |m_N(x)-m(x)| \to 0
\end{equation*}
almost surely.
\end{lemma}
From these lemmas, the proof of Theorem \ref{thm:knn} follows.
\begin{proof}[Theorem \ref{thm:knn}]
We need to show that
\begin{equation*}
\sup_{z_0\in\mathcal{Z}} \left|\sum_i w_{N,i}^1(x_0) \hat{Q}_1(f_0(z_0);y_1^i,x_1^i) - \E[Q_1(f_0(z_0);Y_1,X_1)| X_0=x_0]\right| \to 0
\end{equation*}
a.s. for $x_0$ a.e. The desired result then follows from lemmas \ref{lemma:uniformization} and \ref{lemma:opt}. To begin, we have:
\begin{align*}
&\left| \sum_i w_{N,i}^1(x_0) \hat{Q}_1(f_0(z_0);y_1^i,x_1^i) -  \E[Q_1(f_0(z_0);Y_1,X_1)| X_0=x_0] \right| \\
&\le \left| \sum_i w_{N,i}^1(x_0) Q_1(f_0(z_0);y_1^i,x_1^i) -  \E[Q_1(f_0(z_0);Y_1,X_1)| X_0=x_0] \right| \\
&\;\;\; + \left| \sum_i w_{N,i}^1(x_0) \hat{Q}_1(f_0(z_0);y_1^i,x_1^i) - \sum_i w_{N,i}^1(x_0) Q_1(f_0(z_0);y_1^i,x_1^i)\right|.
\end{align*}
Expanding the second term on the right hand side, and using the fact that $\sum_i w^1_{N,i}(x_0) = 1$ and $w^1_{N,i}(x_0) \ge 0$, we have:
\begin{align*}
&\left| \sum_i w^1_{N,i}(x_0) \left(\hat{Q}_1(f_0(z_0);y_1^i,x_1^i) - Q_1(f_0(z_0);y_1^i,x_1^i)\right)\right| \\
&\;\;\;\le \sum_i w_{N,i}^1(x_0) \left| \hat{Q}_1(f_0(z_0);y_1^i,x_1^i) - Q_1(f_0(z_0);y_1^i,x_1^i)\right|\\
&\;\;\;\le\sup_{x_1\in\mathcal{X}_1} \bigg|\min_{z_1\in W_1}\left(g_1(z_1) + \sum_i w^2_{N,i}(x_1) \hat{Q}_2(f_1(z_1);y_2^i,x_2^i)\right) \\
&\;\;\;\;\;\;\; - \min_{z_1\in W_1}\left(g_1(z_1)+ \E[Q_2(f_1(z_1);Y_2,X_2)| X_1=x_1]\right)\bigg| \\
&\;\;\; \le\sup_{x_1\in\mathcal{X}_1}\sup_{z\in W_1} \left | \sum_{i=1}^N w^2_{N,i}(x_1) \hat{Q}_2(f_1(z);y_2^i,x_2^i) - \E[Q_2(f_1(z);Y_2)| x_1]\right|,
\end{align*}
where we have used lemma \ref{lemma:uniformization}. Therefore, we have:
\begin{align*}
&\left| \sum_i w_{N,i}^1(x_0) \hat{Q}_1(f_0(z_0);y_1^i,x_1^i) -  \E[Q_1(f_0(z_0);Y_1,X_1)|x_0] \right| \\
&\;\;\;\le \left| \sum_i w_{N,i}^1(x_0) Q_1(f_0(z_0);y_1^i,x_1^i) -  \E[Q_1(f_0(z_0);Y_1,X_1)|x_0] \right| \\
&\;\;\;\; +\sup_{x_1\in\mathcal{X}_1}\sup_{z_1\in W_1} \left | \sum_{i=1}^N w^2_{N,i}(x_1) \hat{Q}_2(f_1(z_1);y_2^i,x_2^i) - \E[Q_2(f_1(z_1);Y_2,X_2)|x_1]\right|.
\end{align*}
Repeating the above argument for $t=2,\ldots,T-1$, we have:
\begin{align*}
&\sup_{z_0\in\mathcal{Z}_0}\left| \sum_i w_{N,i}^1(x_0) \hat{Q}_1(f_0(z_0);y_1^i,x_1^i) -  \E[Q_1(f_0(z_0);Y_1,X_1)|X_0=x_0] \right| \\
&\le  \sup_{z_0\in\mathcal{Z}}\left| \sum_i w_{N,i}^1(x_0) Q_1(f_0(z_0);y_1^i,x_1^i) -  \E[Q_1(f_0(z_0);Y_1,X_1)|X_0=x_0] \right| \\
&\;\;+ \sum_{t=1}^{T-1} \sup_{x_t\in\mathcal{X}_t}\sup_{z_t\in W_t}\bigg| \sum_i w_{N,i}^{t+1}(x_t) Q_{t+1}(f_t(z_t);y_{t+1}^i,x_{t+1}^i) \\
&\;\;\;\; -\E[Q_{t+1}(f_t(z_t);Y_{t+1},X_{t+1})| x_t] \bigg|.
\end{align*}

To see that each term on the right hand side goes to 0 a.s., we first apply lemma \ref{lemma:knnconsistency} to each term. This shows that each term (without the supremums over $z_t$) goes to 0 a.s. for each $z_t$. Next we apply lemma \ref{lemma:final} to each term to show that the convergence holds simultaneously for all $z_t$ with probability 1. Finally, we apply lemma \ref{lemma:uniform} to show the convergence of each term is uniform over $z_t$ a.s. To do so, we let $\hat{C}_N(z|x) = \sum_i w_{N,i}^{t+1}(x) Q_{t+1}(f_t(z);y_{t+1}^i,x_{t+1}^i)$ and $C(z|x) =  \E[Q_{t+1}(f_t(z);Y_{t+1},X_{t+1})| x]$. We can verify the equicontinuity assumption holds for each of these functions because of Assumption \ref{as:continuity} and Jensen's inequality (because $w_{N,i}^{t+1}(x)$ define a probability distribution). This completes the proof.
\end{proof}

\subsection{CART Weight Functions}
In order to study the asymptotic properties of (\ref{eq:empirical}) with the CART and random forest weight functions, we need to consider modified versions of the original algorithms of \citet{breiman1984} and \citet{breiman2001}. Since greedy decision trees have proven difficult to analyze theoretically, we instead consider a modified tree learner introduced by \cite{wager2015}. Formally, a regression tree is defined as
\begin{equation*}
T(x;\xi,X_1,Y_1,\ldots,X_n,Y_n) = \frac{1}{|\{i : X_i \in R(x)\}|}\sum_{\{i : X_i \in R(x)\}} Y_i,
\end{equation*}
where $R(x)$ identifies the region of the tree containing $x$, and $\xi$ is an auxiliary source of randomness. Trees are built by recursively partitioning the feature space. At each step of the training process, for each region, a feature is selected and a cutoff is chosen to define an axis-aligned hyperplane to partition the region into two smaller regions. This is repeated until every region contains some minimum number of training points.  In order to guarantee consistency, we place several restrictions on how the trees are built. We use the following definitions from \cite{wager2015}.

\begin{definition}[Random-split, regular, and honest trees]\label{def:tree} Let the regression tree $T(x;\xi,X_1,Y_1,\ldots,X_n,Y_n)$ be the type defined above.
\begin{enumerate}
\item $T$ is a \emph{random-split} tree if at each step in the training procedure, the probability that the next split occurs in the $j$th feature is at least $\pi / d$ for all $j=1,\ldots,d$, with some $\pi > 0$. This source of this randomness is $\xi$.
\item $T$ is a \emph{regular} tree if at each split leaves at least a fraction $\lambda > 0$ of the available training examples on each side of the split. Additionally, the tree is grown to full depth $k \in \mathbb{N}$, meaning there are between $k$ and $2k-1$ training examples in each region of the feature space.
\item $T$ is an \emph{honest} tree if the splits are made independently of the response variables $\{Y_1,\ldots,Y_n\}$. This can be achieved by ignoring the response variable entirely when making splits or by splitting the training data into two halves, one for making splits and one for making predictions. (If the latter is used, the tree is regular if at least a fraction $\lambda$ of the available prediction examples are on each side of the split.)
\end{enumerate}
\end{definition}
The standard implementation of the CART algorithm does not satisfy these definitions, but it is straightforward to modify the original algorithm so that it does. It involves modifying how splits are chosen. If we learn weight functions using trees that do satisfy these definitions, we can guarantee the solutions to (\ref{eq:empirical}) are weakly asymptotically optimal. We first introduce two additional assumptions. We note that these assumptions are strengthened versions of Assumptions \ref{as:continuity} and  \ref{as:uniform}.

\begin{assumption}[Distribution of Auxiliary Covariates]\label{as:rfdist}
The distribution of the auxiliary data, $X$, is uniform on $[0,1]^d$ (independent in each feature)\footnote{The result holds under more general distributional assumptions, but uniformity is assumed for simplicity. For example, we could assume $X$ has a continuous density function on $[0,1]^d$, bounded away from 0 and $\infty$. See \citet{wager2015} for a further discussion.}.
\end{assumption}

\begin{assumption}[Continuity]\label{as2:continuity}
For each $t=1,\ldots,T$, there exists an $L_t < \infty$ such that $\forall y_t \in \mathcal{Y}_t,z_t\in W_t$ and $\forall z,z' \in W_{t-1}$,
\begin{equation*}
|Q_t(f_{t-1}(z);y_t,x_t) - Q_t(f_{t-1}(z');y_t,x_t)| \le L_t||z-z'||. 
\end{equation*}
Furthermore, for each $t=1,\ldots,T$, $\E[Q_t(s_t;Y_t,X_t)|X=x]$ is $M_t$-Lipschitz continuous in $x$, for all $s_t$.
\end{assumption}

\begin{theorem}\label{thm:cart}
Suppose Assumptions \ref{as:regularity}-\ref{as2:continuity} hold, and the training data is i.i.d. Let $w_{N,i}^t(x_{t-1})$ be the CART weight functions for $t=1,\ldots,T$, and assume the trees are honest, random split, and regular with $k$, the minimum number of training examples in each leaf, equal to $\min\{\lceil CN^\delta\rceil,N-1\}$ for $C > 0, \delta \in (0,1)$.  Then $\{\hat{z}_0^N(x)\}$, a sequence of optimal solutions to (\ref{eq:empirical}), is weakly asymptotically optimal.
\end{theorem}

The proof of this result follows closely the proof of Theorem \ref{thm:knn}. To begin, we prove a result regarding the bias of tree based predictors. It relies on the same argument \citet{wager2015} used in proving their Lemma 2.

\begin{lemma}\label{lemma:treebias}
Suppose $T$ is a regular, random-split tree as in Definition \ref{def:tree}, and the training covariates $X_1,\ldots,X_N$ are i.i.d. uniform($[0,1]^d$) random variables. If $R(x)$ denotes the partition of $[0,1]^d$ containing $x \in \R^d$, and $\dfrac{N}{k} \to \infty$ as $N\to\infty$, then
\begin{equation*}
\sup_x \text{diam}(R(x)) \to_P 0.
\end{equation*}
\end{lemma}
\begin{proof}
As in the proof of Lemma 2 in \citet{wager2015}, we define $c(x)$ to be the number of splits leading to the leaf $R(x)$ and $c_j(x)$ to be the number of these splits that are on the $j$th coordinate. Following the same arguments as \citet{wager2015}, we have, conditional on $X_1,\ldots,X_N$,
\begin{equation*}
c_j(x) \ge B_j(x),
\end{equation*}
where $B_j(x) \sim \text{Binom}\left(\left\lfloor\dfrac{\log(N/(2k-1))}{\log \lambda^{-1}}\right\rfloor,\dfrac{\pi}{d}\right)$. In addition, for $N$ sufficiently large, with probability at least $1-k/N$ (over the training data),
\begin{equation*}
\text{diam}_j(R(x)) \le (1-\lambda)^{0.99 c_j(x)}.
\end{equation*}
We call this event $A_N$. From here, we have, for any $\epsilon > 0$,
\begin{align*}
&P\left(\sup_x \text{diam}(R(x)) > \epsilon\bigg| X_1,\ldots,X_N,A_N\right) \\
&\le P\left(0.99 \inf_j \inf_x c_j(x) < \frac{\log \epsilon^{-1}}{\log(1-\lambda)^{-1}}\bigg| X_1,\ldots,X_N,A_N\right) \\
&\le d\left(\frac{N}{k}\right)P\left(c_j(x) \le \frac{\log \epsilon^{-1}}{0.99\log(1-\lambda)^{-1}}\bigg| X_1,\ldots,X_N,A_N\right) \\
&\le d\left(\frac{N}{k}\right)P\left(B_j(x) \le \frac{\log \epsilon^{-1}}{0.99\log(1-\lambda)^{-1}}\bigg| X_1,\ldots,X_N,A_N\right) \\
&\le d\left(\frac{N}{k}\right)\exp\left(-2\left\lfloor\dfrac{\log\frac{N}{2k-1}}{\log \lambda^{-1}}\right\rfloor\left(\frac{\pi}{d}\left\lfloor\dfrac{\log\frac{N}{2k-1}}{\log \lambda^{-1}}\right\rfloor -  \frac{\log \epsilon^{-1}}{0.99\log(1-\lambda)^{-1}}\right)^2 \right) \\
&\le d\left(\frac{N}{k}\right)\exp\left(-C_1 \log^3(N/(2k-1))\right),
\end{align*}
for $N/k$ sufficiently large, where $C_1 > 0$ is a constant that does not depend on $N$ or $k$. The second inequality follows from the union bound since there are a maximum of $N/k$ total partitions in the tree, and the fourth inequality follows from Hoeffding's inequality. Putting everything together, we have:
\begin{align*}
&P\left(\sup_x \text{diam}(R(x)) > \epsilon\right) \\
&\le P\left(\sup_x \text{diam}(R(x)) > \epsilon\bigg| A_N \right) + \frac{k}{N} \\
&\le d\left(\frac{N}{k}\right)\exp\left(-C_1 \log^3(N/(2k-1))\right) + \frac{k}{N}.
\end{align*}
It is easy to verify that the final expression goes to 0 as $N \to \infty$, so the proof is complete.
\end{proof}
Next, we establish the uniform consistency of the CART regression estimator.
\begin{lemma}\label{lemma:treeconvergence}
Suppose $T$ is a regular, random-split, honest tree as in Definition \ref{def:tree}, the training data $(X_1,Y_1),\ldots,(X_N,Y_N)$ are i.i.d. with $X_i$ uniform on $[0,1]^d$, and $\E[Y|X=x]$ is $L$-Lipschitz continuous. In addition, assume there exists $\lambda > 0$ such that the uniform noise condition is satisfied: $\sup\limits_x \E[\exp(\lambda|Y_i-\E[Y_i|X_i=x]|)|X_i=x] < \infty$. Finally, suppose $\log N /k \to 0$ and $N/k \to \infty$ as $N\to\infty$. If $\hat{\mu}_N(x)$ denotes the prediction of $T$ at $X\in\R^d$ and $\mu(x) = \E[Y|X=x]$, then
\begin{equation*}
\sup_{x} |\hat{\mu}_N(x) - \mu(x)| \to_P 0.
\end{equation*}
\end{lemma}
\begin{proof}
To begin, we have
\begin{align*}
&\sup_{x} |\hat{\mu}_N(x) - \mu(x)| \\
&\le \sup_x |\hat{\mu}_N(x) - \E[\hat{\mu}_N(x)|X_1,\ldots,X_N]| + \sup_x |\E[\hat{\mu}_N(x)|X_1,\ldots,X_N] - \mu(x)| \\
&\le  \sup_x \left|\sum_i w_{N,i}(x)(Y_i -\E[Y_i|X_i])\right|+ \sup_x \left|\sum_i w_{N,i}(x)(\mu(X_i) - \mu(x))\right| \\
&\le  \sup_x \left|\sum_i w_{N,i}(x)(Y_i -\E[Y_i|X_i])\right|+ \sup_x \sum_i w_{N,i}(x)\left|\mu(X_i) - \mu(x)\right| \\
&\le  \sup_x \left|\sum_i w_{N,i}(x)(Y_i -\E[Y_i|X_i])\right| + L\sup_x \sum_i w_{N,i}(x) ||X_i - x|| \\
&\le  \sup_x \left|\sum_i w_{N,i}(x)(Y_i -\E[Y_i|X_i])\right| + L\sup_x \text{diam}(R(x)),
\end{align*}
where $w_{N,i}(x)$ is the CART weight function corresponding to tree $T$. In the third and fourth inequalities we used Jensen's inequality and the Lipschitz continuity assumption. By lemma \ref{lemma:treebias}, the latter term goes to 0 in probability. For the former, we define $M(x) = |\{i : X_i \in R(x)\}|$ to be the number of training examples in the leaf containing $x$. For fixed $x$, if $c = \sup\limits_x \E[\exp(\lambda|Y_i-\E[Y_i|X_i=x]|)|X_i=x]$, then by Lemma 12.1 of \citet{biau2015}, we have, for any $\epsilon > 0$,
\begin{align*}
&P\left(\left|\sum_i w_{N,i}(x)(Y_i - \E[Y_i|X_i])\right| > \epsilon \bigg| X_1,\ldots,X_N\right) \\
&\le 2\exp\left(-M(x)\frac{\min(\epsilon,\min(1,2c)/\lambda)^2\lambda^2}{8c}\right).
\end{align*}
Because there are a maximum of $N/k$ leaves in the tree, there are a maximum of $N/k$ values of the weight function $w_{N,i}(x)$. Therefore, we can use the union bound to show:
\begin{align*}
&P\left(\sup_x\left|\sum_i w_{N,i}(x)(Y_i - \E[Y_i|X_i])\right| > \epsilon \bigg| X_1,\ldots,X_N\right) \\
&\le \frac{2N}{k}\exp\left(-M(x)\frac{\min(\epsilon,\min(1,2c)/\lambda)^2\lambda^2}{8c}\right) \\
&\le \frac{2N}{k}\exp\left(-k\frac{\min(\epsilon,\min(1,2c)/\lambda)^2\lambda^2}{8c}\right) \\
&= \exp\left(\log(2N) - \log k-k\frac{\min(\epsilon,\min(1,2c)/\lambda)^2\lambda^2}{8c}\right)
\end{align*}
where the second inequality follows because $M(x) \ge k$ by assumption. Taking the expectation of both sides and the limit as $N\to\infty$ completes the proof.
\end{proof}
We prove one more intermediate result, and the proof of Theorem \ref{thm:cart} will follow.
\begin{lemma}\label{lemma:netuniformization}
Suppose $h(z;y)$ is an $L$ Lipschitz continuous function for all $y$ (with respect to $||\cdot||_p$), and $\mathcal{Z}\subset\R^d$ is nonempty, closed, and bounded with diameter $D$. It follows that
\begin{equation*}
P\left(\sup_{z\in\mathcal{Z}} h(z;Y) > \epsilon\right) \le \left(\frac{2\rho D L}{\epsilon}\right)^d \sup_{z\in\mathcal{Z}} P\left(h(z;Y) > \frac{\epsilon}{2}\right),
\end{equation*}
where $\rho > 0$ is a constant that depends only on $p$.
\end{lemma}
\begin{proof}
This result follows from a standard covering number argument. We can construct a $\nu$-net of $\mathcal{Z}$, $z_1,\ldots,z_K$ with $K \le \left(\dfrac{\rho D}{\nu}\right)^d$. That is, $\forall z\in\mathcal{Z}$, there exists $i$ such that $||z-z_i|| \le \nu$. If we define $i(z)$ to be the function that returns this index, then, for all $y$,
\begin{align*}
h(z;y) \le h(z_{i(z)},y) + |h(z;y) - h(z_{i(z)};y)| &\le  h(z_{i(z)},y) + L||z - z_{i(z)}|| \\
&\le  h(z_{i(z)},y) + L\nu.
\end{align*}
Taking the supremum of both sides over $z \in \mathcal{Z}$, we have
\begin{equation*}
\sup_{z\in\mathcal{Z}} h(z;y) \le \max_{i=1,\ldots,K} h(z_i;y) + L\nu.
\end{equation*}
Next, we select $\nu = \dfrac{\epsilon}{2L}$, and we have:
\begin{align*}
P\left(\sup_{z\in\mathcal{Z}} h(z;Y) > \epsilon\right) &\le P\left(\max_{i=1,\ldots,K} h(z_i;Y) > \frac{\epsilon}{2}\right) \\
&\le \left(\frac{2\rho D L}{\epsilon}\right)^d \sup_{z\in\mathcal{Z}} P\left(h(z;Y) > \frac{\epsilon}{2}\right),
\end{align*}
where the final inequality follows from the union bound.
\end{proof}

\begin{proof}[Theorem \ref{thm:cart}]
The proof follows the same outline as the proof of Theorem \ref{thm:knn}. We need to show
\begin{equation*}
\sup_{z_0\in\mathcal{Z}} \left|\sum_i w_{N,i}^1(x_0) \hat{Q}_1(f_0(z_0);y_1^i,x_1^i) - \E[Q_1(f_0(z_0);Y_1,X_1)| X_0=x_0]\right| \to_P 0,
\end{equation*}
for $x_0$ a.e. The desired result then follows from lemma \ref{lemma:uniformization}. Following the same steps as in the proof of Theorem \ref{thm:knn}, we have:
\begin{align*}
&\sup_{z_0\in\mathcal{Z}_0}\left| \sum_i w_{N,i}^1(x_0) \hat{Q}_1(f_0(z_0);y_1^i,x_1^i) -  \E[Q_1(f_0(z_0);Y_1,X_1)|X_0=x_0] \right| \\
&\le  \sup_{z_0\in\mathcal{Z}}\left| \sum_i w_{N,i}^1(x_0) Q_1(f_0(z_0);y_1^i,x_1^i) -  \E[Q_1(f_0(z_0);Y_1,X_1)|X_0=x_0] \right| \\
&\;\;\;+ \sum_{t=1}^{T-1} \sup_{x_t\in\mathcal{X}_t}\sup_{z_t\in W_t}\bigg| \sum_i w_{N,i}^{t+1}(x_t) Q_{t+1}(f_t(z_t);y_{t+1}^i,x_{t+1}^i) \\
&\;\;\;\;\;-  \E[Q_{t+1}(f_t(z_t);Y_{t+1},X_{t+1})| x_t] \bigg|.
\end{align*}
Next, we have, for all $x_t$,
\begin{align*}
&\left| \sum_i w_{N,i}^{t+1}(x_t) Q_{t+1}(f_t(z);y_{t+1}^i,x_{t+1}^i) -  \E[Q_{t+1}(f_t(z);Y_{t+1},X_{t+1})| x_t] \right| \\
&- \left| \sum_i w_{N,i}^{t+1}(x_t) Q_{t+1}(f_t(z');y_{t+1}^i,x_{t+1}^i) -  \E[Q_{t+1}(f_t(z');Y_{t+1},X_{t+1})| x_t] \right| \\
&\le \left|  \sum_i w_{N,i}^{t+1}(x_t) (Q_{t+1}(f_t(z);y_{t+1}^i,x_{t+1}^i) - Q_{t+1}(f_t(z');y_{t+1}^i,x_{t+1}^i))\right| \\
&\;\;\;+ \left| \E[Q_{t+1}(f_t(z);Y_{t+1},X_{t+1}) - Q_{t+1}(f_t(z');Y_{t+1},X_{t+1})| x_t] \right| \\
&\le 2\sup_{x_{t+1},y_{t+1}}\left| (Q_{t+1}(f_t(z);y_{t+1},x_{t+1}) - Q_{t+1}(f_t(z');y_{t+1},x_{t+1}))\right|\\
&\le 2L_{t+1} ||z-z'||.
\end{align*}
Rearranging, and taking the supremum over both sides, we have
\begin{align*}
&\sup_{x_t\in\mathcal{X}_t}\left| \sum_i w_{N,i}^{t+1}(x_t) Q_{t+1}(f_t(z);y_{t+1}^i,x_{t+1}^i) -  \E[Q_{t+1}(f_t(z);Y_{t+1},X_{t+1})| x_t] \right| \\
&\le \sup_{x_t\in\mathcal{X}_t}\bigg| \sum_i w_{N,i}^{t+1}(x_t) Q_{t+1}(f_t(z');y_{t+1}^i,x_{t+1}^i) \\
&\;\;\;\;\;-  \E[Q_{t+1}(f_t(z');Y_{t+1},X_{t+1})| x_t] \bigg| + 2L_{t+1} ||z-z'||,
\end{align*}
which demonstrates 
\begin{equation*}
\sup\limits_{x_t\in\mathcal{X}_t}\left| \sum_i w_{N,i}^{t+1}(x_t) Q_{t+1}(f_t(z);y_{t+1}^i,x_{t+1}^i) -  \E[Q_{t+1}(f_t(z);Y_{t+1},X_{t+1})| x_t] \right|
\end{equation*}
 is $2L_{t+1}$ Lipschitz. We now apply lemma \ref{lemma:netuniformization} to get:
\begin{align*}
&P\bigg(\sup_{z_t\in W_t}\sup_{x_t\in\mathcal{X}_t}\bigg| \sum_i w_{N,i}^{t+1}(x_t) Q_{t+1}(f_t(z);y_{t+1}^i,x_{t+1}^i)\\
&\;\;\;\; -  \E[Q_{t+1}(f_t(z);Y_{t+1},X_{t+1})| x_t] \bigg| > \epsilon\bigg) \\
&\le \left(\frac{2\rho D_tL_{t+1}}{\epsilon}\right)^{p_t}P\bigg(\sup_{x_t\in\mathcal{X}_t}\bigg| \sum_i w_{N,i}^{t+1}(x_t) Q_{t+1}(f_t(z);y_{t+1}^i,x_{t+1}^i) \\
&\;\;\;\; -  \E[Q_{t+1}(f_t(z);Y_{t+1},X_{t+1})| x_t] \bigg| > \frac{\epsilon}{2}\bigg).
\end{align*}
By lemma \ref{lemma:treeconvergence}, the right hand side goes to 0 in probability. Repeating this argument for all $t$ completes the result.
\end{proof}

\subsection{Random Forest Weight Functions}
A random forest is an ensemble method that aggregates regression trees as base learners in order to make predictions. To aggregate the trees into a random forest, Breiman suggested training each tree on a bootstrapped sample of the training examples. In order to facilitate the theoretical analysis, we instead build a forest by training trees on subsamples of size $s_N$ of the training data. The random forest estimator is given by:
\begin{equation}\label{eq:rfdef}
R(x;\xi,X_1,Y_1,\ldots,X_N,Y_N) = \frac{1}{B} \sum_{b=1}^B T(x;\xi_b,X_{S_b},Y_{S_b}),
\end{equation}
where $(X_{S_b},Y_{S_b})$ denotes a random subset of training examples of size $s_N$. Given this definition of a random forest, we have the following asymptotic optimality result for random forest weight functions.

\begin{theorem}\label{thm:rf}
Suppose Assumptions \ref{as:regularity}-\ref{as2:continuity} hold, and the training data is i.i.d. Let $w_{N,i}^t(x_{t-1})$ be the random forest weight functions for $t=1,\ldots,T$. Assume the trees that make up the forest are honest, random split, and regular and $k$, the minimum number of training examples in each leaf, equals $\min\{\lceil C_1N^\delta\rceil,N-1\}$ for $C_1 > 0, \delta \in (0,1)$. Furthermore, assume $s_N$, the subsample size, equals $\min\{\lceil C_2N^\alpha\rceil,N-1\}$ for $C_2 > 0, \alpha \in (\delta,1)$. Then $\{\hat{z}_0^N(x)\}$, a sequence of optimal solutions to (\ref{eq:empirical}), is weakly asymptotically optimal.
\end{theorem}
\begin{proof}
The proof exactly mirrors the proof of Theorem \ref{thm:cart}, except we use lemma \ref{lemma:rfconsistency} in place of lemma \ref{lemma:treeconvergence}.
\end{proof}

This theorem shows it is possible to obtain asymptotically optimal solutions to (\ref{eq:empirical}) with random forest weight functions. The random forest model we use is slightly different than Breiman's original algorithm, which is implemented in common machine learning libraries. For example, we require that $k$, the minimum number of training examples in each leaf, grows with $N$, whereas the original algorithm has $k$ fixed at 1. We include an additional theorem in the appendix which proves the strong asymptotic optimality of random forest weight functions with $k$ fixed for the single stage version of the problem. However, the proof does not extend to the multistage problem we consider here. The proof of Theorem \ref{thm:rf} uses the following lemma.

\begin{lemma}\label{lemma:rfconsistency}
Suppose $R$ is a random forest consisting of $B$ regular, random-split, honest trees, each trained on a random subset of the training data of size $s_N$. Assume the training data $(X_1,Y_1),\ldots,(X_N,Y_N)$ are i.i.d. with $X_i$ uniform on $[0,1]^d$, $\E[Y|X=x]$ is $L$-Lipschitz continuous, and there exists $\lambda > 0$ such that the uniform noise condition is satisfied: $\sup\limits_x \E[\exp(\lambda|Y_i-\E[Y_i|X_i=x]|)|X_i=x] < \infty$. Finally, suppose $\log s_N /k_N \to 0$ and $s_N/k_N \to \infty$ as $N\to\infty$. If $\hat{\mu}_N(x)$ denotes the prediction of $R$ at $X\in\R^d$ and $\mu(x) = \E[Y|X=x]$, then
\begin{equation*}
\sup_{x} |\hat{\mu}_N(x) - \mu(x)| \to_P 0.
\end{equation*}
\end{lemma}
\begin{proof}
We define the prediction of the $b$th tree in the ensemble to equal $\hat{\mu}_N^b(x)$, so we have, by Jensen's inequality,
\begin{align*}
\sup_x|\hat{\mu}_N(x) - \mu(x)| &\le \frac{1}{B}\sum_{b=1}^B \sup_x|\hat{\mu}_N^b(x) -\mu(x)|.
\end{align*}
By lemma \ref{lemma:treeconvergence}, we immediately have that each term on the right hand side goes to 0 in probability. Because $B$ does not depend on $N$, this completes the result.
\end{proof}

\section{Finite Sample Guarantees}\label{sec:finitesample}
In this section, we establish finite sample, probabilistic guarantees for the difference between the cost of a solution to (\ref{eq:empirical}) and the true optimal cost. We focus on $k$-nearest neighbor weight functions. To the best of our knowledge, this is the first finite sample bound for either the single-stage or multistage setting with auxiliary data.

To facilitate the presentation of our result, we begin by discussing convergence rate results for the single stage setting. Without auxiliary data, the problem we want to solve is given by
\begin{equation*}
\min_{z\in\mathcal{Z}} \E[c(z;Y)].
\end{equation*}
If $\hat{z}_N$ represents the SAA approach applied to this problem, then, under appropriate conditions, the regret, $\E[c(\hat{z}_N,Y)] - \min\limits_{z\in\mathcal{Z}}\E[c(z;Y)]$, is $\widetilde{O}_p\left(\dfrac{1}{\sqrt{N}}\right)$, where the $\widetilde{O}_p$ notation suppresses logarithmic dependencies (see, for example, \citep{shapiro2014lectures}). This implies that for any confidence level, $\alpha$, we know that the regret is bounded by a term of order $1/\sqrt{N}$ with probability at least $1-\alpha$. We contrast this with the setting in which we have auxiliary data,
\begin{equation*}
\min_{z\in\mathcal{Z}} \E[c(z;Y)|X=x].
\end{equation*}
If $\hat{z}_N(x)$ represents a solution to this problem using the approach of \citet{bertsimas2014} with the $k$-nearest neighbor weight functions, then the regret, $\E[c(\hat{z}_N(x),Y)|X=x] - \min\limits_{x\in\mathcal{X}} \E[c(z;Y)|X=x]$ is $\widetilde{O}_p\left(\dfrac{1}{N^{1/2d}}\right)$ for $d\ge 2$, where $d$ is the dimension of the auxiliary covariate space. The problem with auxiliary data is clearly harder. The baseline with respect to which we compute the regret is smaller because it takes into account the value of the auxiliary covariates. Furthermore, many of the $x^i$s in the training data will be very different from the $x$ we are concerned with. In fact, with the $k$-nearest neighbor weight functions, we effectively throw out all but the $k$ most relevant training examples. Because of this, we pay a penalty that depends on the dimension of the auxiliary covariate space.

To formalize the above discussion for the multistage setting, we add two additional assumptions.
\begin{assumption}[Distribution of auxiliary covariates]\label{as2:auxcovariates}
For each $t$, $X_t$ has its support contained in $[0,1]^{d_t}$ and $\forall x_t \in \text{support}(X_t)$ and $\forall \epsilon > 0$, $P(X \in B_\epsilon(x_t)) > g \epsilon^d$ with $g > 0$, where $B_\epsilon(x) = \{x' : ||x - x'|| \le \epsilon\}$.
\end{assumption}
This assumption is satisfied, for example, if $X_t$ is uniformly distributed or has finite support and, thus, is more general than Assumption \ref{as:rfdist}.

\begin{assumption}[Subgaussian noise terms]\label{as:subgaussian}
The noise terms are uniformly subgaussian, i.e., defining $N_t(s_t;x_{t-1}) = Q_t(s_t;Y_t,X_t) - \E[Q_t(s_t;Y_t,X_t)|X_{t-1}=x_{t-1}]$, there exists $\sigma^2 > 0$ such that
\begin{equation*}
\max_{t=1,\ldots,T} \sup_{x_{t-1}\in\mathcal{X}_{t-1}}\sup_{s_t\in S_t} \E[e^{\lambda N_t(s_t;x_{t-1})}|X_{t-1}=x_{t-1}] \le \exp(\lambda\sigma^2/2),
\end{equation*}
for all $\lambda > 0$.
\end{assumption}
This assumption implies condition 4 of Assumption \ref{as:uniform}. With these additional assumptions, we have the following theorem.

\begin{theorem}\label{thm:knnfinitesample}
Suppose Assumptions \ref{as:regularity}-\ref{as:uniform} and Assumptions \ref{as2:continuity}-\ref{as:subgaussian} hold, the training data, $(X_1,Y_1),\ldots,(X_N,Y_N)$, is i.i.d., and $w^t_{N,i}(x_t)$ are the $k_N$ nearest neighbor weight functions with $k_N = \min\{\lceil CN^\delta\rceil,N-1\}$ for $C > 0$ and $\delta \in(0,1)$. We define
\begin{equation*}
R_N(x_0) = g_0(\hat{z}_0^N) + \E[g_0(\hat{z}_0^N) + Q_1(f_0(\hat{z}_0^N);Y_1,X_1)|X_0=x_0] - v^*(x_0),
\end{equation*}
where $\hat{z}_0^N$ is an optimal solution to (\ref{eq:empirical}). For any $\alpha\in(0,1)$, with probability at least $1-\alpha$,
\begin{align*}
&R_N(x_0) \le \frac{TC_1}{N^{\delta/2}}\left(\sqrt{\log\frac{1}{\alpha}} + \sqrt{2d\log N} + C_2\right) \\
&\;\;\;+TC_3\left(\left(\frac{1}{N}\log\frac{1}{\alpha}\right)^{1/2d} +C_4N^{(\delta-1)/d} + \frac{C_5}{N^{1/2d}}\right),
\end{align*}
for $N \ge 25d$, for $x_0$ almost everywhere. (Here, $d= \max_t d_t$.) $C_1$, $C_2$, $C_3$, $C_4$, and $C_5$ are constants that may depend only logarithmically on $T$ and $\log\frac{1}{\alpha}$ and are defined in (\ref{eq:finiteconstants}) in the proof.
\end{theorem}
$R_N(x_0)$ represents the regret of the solution to (\ref{eq:empirical}), i.e., the difference between the cost of the solution and the true optimal cost, $v^*(x_0)$. We can optimize the bound by choosing $\delta=\dfrac{2}{2+d}$ and restate the result more prosaically:
\begin{equation*}
R_N(x_0) = \begin{cases}
\widetilde{O}_p\left(\dfrac{T}{N^{1/3}}\right), & d=1,\\\\
\widetilde{O}_p\left(\dfrac{T}{N^{1/2d}}\right), & d \ge 2.
\end{cases}
\end{equation*}
As before, this result is best understood in comparison with the multistage SAA problem without any auxiliary covariates. For this problem, regret is $\widetilde{O}_p\left(\dfrac{T}{\sqrt{N}}\right)$ \citep[ch. 5]{shapiro2014lectures}. We pay a penalty that depends on $d$, the maximum dimension of the auxiliary covariate spaces, $\mathcal{X}_t$.

To prove this result, we rely on the following lemma, which provides a finite sample guarantee on the error of the $k$NN regression estimator.
\begin{lemma}\label{lemma:knnrate}
Suppose $X$ has support $\mathcal{X}\subset [0,1]^d$, $P(X\in B_\epsilon(x)) > g\epsilon^d$ for all $x\in\mathcal{X}$, and $Y - \E[Y|X=x]$ is conditionally subgaussian given $X=x$ with variance proxy $\sigma^2$, uniformly for all $x\in\mathcal{X}$. Assume the training data, $(X_1,Y_1),\ldots,(X_N,Y_N)$ is i.i.d. and that $\E[Y|X=x]$ is $L$-Lipschitz. If $\hat{\mu}_N(x)$ denotes the $k$NN regression estimator at $x$ and $\mu(x) = \E[Y|X=x]$, then
\begin{align*}
&P\left(\sup\limits_{x\in\mathcal{X}}|\hat{\mu}_N(x) - \mu(x)| > \epsilon\right) \\
&\le \left(\dfrac{4\sqrt{d}\rho L}{\epsilon}\right)^d\exp\left(-\frac{2}{N}\left(Ng\left(\frac{\epsilon}{4L}\right)^d+1-k\right)^2\right) \\
&\;\;\;\;+ 2\left(\dfrac{25}{d}\right)^dN^{2d}\exp\left(-\frac{k\epsilon^2}{8\sigma^2}\right),
\end{align*}
for any $\epsilon \ge 2L\left(\dfrac{k-1}{Ng}\right)^{1/d}$ and $N \ge 2d$.
\end{lemma}
\begin{proof}[Lemma \ref{lemma:knnrate}]
We decompose $\sup\limits_{x\in\mathcal{X}}|\hat{\mu}_N(x) - \mu(x)|$ into a sum of two terms: $\sup\limits_{x\in\mathcal{X}}|\hat{\mu}_N(x) - \E[\hat{\mu}_N(x)|X_1,\ldots,X_N]|$ and $\sup\limits_{x\in\mathcal{X}}|\mu(x) - \E[\hat{\mu}_N(x)|X_1,\ldots,X_N]|$. For the latter term, we utilize the Lipschitz assumption to show
\begin{align*}
\sup\limits_{x\in\mathcal{X}}|\mu(x) - \E[\hat{\mu}_N(x)|X_1,\ldots,X_N]| &\le L \sup_{x\in\mathcal{X}} \frac{1}{k}\sum_{i=1}^k ||X_{(i)}(x) - x||,
\end{align*}
where $X_{(i)}(x)$ denotes the $i$th nearest neighbor of $x$ out of $X_1,\ldots,X_N$, as measured by Euclidean distance. Next, we note that $\frac{1}{k}\sum_{i=1}^k ||X_{(i)}(x) - x|| \le ||X_{(k)}(x) - x|| $. This gives:
\begin{align*}
&P\left(\sup\limits_{x\in\mathcal{X}}|\mu(x) - \E[\hat{\mu}_N(x)|X_1,\ldots,X_N]| > L\epsilon \right) \\
&\le P\left(\sup_{x\in\mathcal{X}}||X_{(k)}(x)-x|| > \epsilon\right) \\
&= P\left(\inf_{x\in\mathcal{X}} |\{i : X_i \in B_{\epsilon}(x)| \le k-1\right).
\end{align*}
Next, we construct an $\epsilon/2$-net for $\mathcal{X}$, $\hat{x}_1,\ldots,\hat{x}_m$, with $m \le \left(\dfrac{2\sqrt{d}\rho}{\epsilon}\right)^d$. For any $x\in\mathcal{X}$, there exists a $j$ such that $B_{\epsilon/2}(\hat{x}_j) \subset B_\epsilon(x)$. Therefore, we can upper bound the above expression by
\begin{align*}
&P\left(\min_{j=1,\ldots,m}|\{i : X_i \in B_{\epsilon/2}(\hat{x}_j)| \le k-1\right) \\
&\le \left(\dfrac{2\sqrt{d}\rho}{\epsilon}\right)^d P(B \le k-1),
\end{align*}
where $B \sim \text{Binom}(N,g(\epsilon/2)^d)$. Applying Hoeffding's bound, we have:
\begin{align*}
&P\left(\sup\limits_{x\in\mathcal{X}}|\mu(x) - \E[\hat{\mu}_N(x)|X_1,\ldots,X_N]| > \epsilon \right) \\
&\;\;\;\le \left(\dfrac{2\sqrt{d}\rho L}{\epsilon}\right)^d\exp\left(-\frac{2}{N}\left(Ng(\epsilon/2L)^d+1-k\right)^2\right),
\end{align*}
for any $\epsilon \ge 2L\left(\dfrac{k-1}{Ng}\right)^{1/d}$.

For the second part, we use Theorem 12.2 of \citet{biau2015}, which says the number of possible distinct orderings of neighbors of $X_1,\ldots,X_N \in \R^d$ is less than or equal to $\left(\dfrac{25}{d}\right)^dN^{2d}$ for all $N \ge 2d$. Therefore,
\begin{align*}
&P\left(\sup_{x\in\mathcal{X}}|\hat{\mu}_N(x) - \E[\hat{\mu}_N(x)|X_1,\ldots,X_N]| > \epsilon\bigg| X_1,\ldots,X_N\right) \\
&\le \left(\dfrac{25}{d}\right)^dN^{2d}\sup_{x\in\mathcal{X}}P\left(|\hat{\mu}_N(x) - \E[\hat{\mu}_N(x)|X_1,\ldots,X_N]| > \epsilon\bigg| X_1,\ldots,X_N\right) \\
&= \left(\dfrac{25}{d}\right)^dN^{2d}\sup_{x\in\mathcal{X}}P\left(\left|\frac{1}{k}\sum_{i=1}^k (Y_{(i)}(x)-\mu(X_{(i)}(x)))\right| > \epsilon\bigg| X_1,\ldots,X_N\right),
\end{align*}
where $Y_{(i)}$ denotes the observation $Y$ corresponding to $X_{(i)}(x)$. It is easy to verify, see Proposition 8.1 of \citet{biau2015} for example, that $Y_{(1)}(x)-\mu(X_{(1)}(x)),\ldots,Y_{(k)}(x)-\mu(X_{(k)}(x))$ are conditionally independent given $X_1,\ldots,X_N$. Therefore, we apply Hoeffding's bound for sums of subgaussian random variables to get, for any $\epsilon > 0$,
\begin{align*}
&P\left(\sup_{x\in\mathcal{X}}|\hat{\mu}_N(x) - \E[\hat{\mu}_N(x)|X_1,\ldots,X_N]| > \epsilon\bigg| X_1,\ldots,X_N\right) \\
&\le 2\left(\dfrac{25}{d}\right)^dN^{2d}\exp\left(-\frac{k\epsilon^2}{2\sigma^2}\right).
\end{align*}
Taking the expectation of both sides and combining with the previous part completes the result.
\end{proof}
Now, we can prove the main result.

\begin{proof}[Theorem \ref{thm:knnfinitesample}]
By lemma \ref{lemma:uniformization}, the regret is bounded by
\begin{align*}
2\sup_{z_0\in\mathcal{Z}} \left|\sum_i w_{N,i}^1(x_0) \hat{Q}_1(f_0(z_0);y_1^i,x_1^i) - \E[Q_1(f_0(z_0);Y_1,X_1)| X_0=x_0]\right|.
\end{align*}
Following the same steps as in the proof of Theorem \ref{thm:knn}, we have:
\begin{align*}
&2\sup_{z_0\in\mathcal{Z}_0}\left| \sum_i w_{N,i}^1(x_0) \hat{Q}_1(f_0(z_0);y_1^i,x_1^i) -  \E[Q_1(f_0(z_0);Y_1,X_1)|X_0=x_0] \right| \\
&\le 2 \sup_{z_0\in\mathcal{Z}}\left| \sum_i w_{N,i}^1(x_0) Q_1(f_0(z_0);y_1^i,x_1^i) -  \E[Q_1(f_0(z_0);Y_1,X_1)|X_0=x_0] \right| \\
&\;\;\;+ 2\sum_{t=1}^{T-1} \sup_{x_t\in\mathcal{X}_t}\sup_{z_t\in W_t}\bigg| \sum_i w_{N,i}^{t+1}(x_t) Q_{t+1}(f_t(z_t);y_{t+1}^i,x_{t+1}^i)\\
&\;\;\;\; -  \E[Q_{t+1}(f_t(z_t);Y_{t+1},X_{t+1})| x_t] \bigg|.
\end{align*}
We next apply lemma \ref{lemma:netuniformization}, as in the proof of Theorem \ref{thm:cart}, to see, for each $t$,
\begin{align*}
&P\bigg(2\sup_{z_t\in W_t}\sup_{x_t\in\mathcal{X}_t}\bigg| \sum_i w_{N,i}^{t+1}(x_t) Q_{t+1}(f_t(z);y_{t+1}^i,x_{t+1}^i) \\
&\;\;\;\;-  \E[Q_{t+1}(f_t(z);Y_{t+1},X_{t+1})| x_t] \bigg| > \frac{\epsilon}{T}\bigg) \\
&\le \left(\frac{4T\rho D_tL_t}{\epsilon}\right)^{p_t}P\bigg(\sup_{x_t\in\mathcal{X}_t}\bigg| \sum_i w_{N,i}^{t+1}(x_t) Q_{t+1}(f_t(z);y_{t+1}^i,x_{t+1}^i)\\
&\;\;\;\;\; -  \E[Q_{t+1}(f_t(z);Y_{t+1},X_{t+1})| x_t] \bigg| > \frac{\epsilon}{4T}\bigg).
\end{align*}
Next, we use lemma \ref{lemma:knnrate} to upper bound this expression by
\begin{align*}
&\left(\frac{4T\rho D_tL_t}{\epsilon}\right)^{p_t}\Bigg[\left(\dfrac{16T\sqrt{d_t}\rho M_t}{\epsilon}\right)^{d_t}\exp\left(-\frac{2}{N}\left(Ng\left(\frac{\epsilon}{16TM_t}\right)^{d_t}+1-k_N\right)^2\right) \\
&\;\;\;\; + 2\left(\dfrac{25}{d_t}\right)^{d_t}N^{2d_t}\exp\left(-\frac{k_N\epsilon^2}{64T^2\sigma^2}\right)\Bigg],
\end{align*}
for any $\epsilon \ge 8TM_t\left(\dfrac{k_N-1}{Ng}\right)^{1/d_t}$ and $N\ge 2d_t$. Combining the results for $t=0,\ldots,T-1$ with the union bound, and plugging in for the definitions of $k_N$, $d$, $p$, $L$, $M$, and $D$, we have:
\begin{align*}
&P\left(2\sup_{z_0\in\mathcal{Z}} \left|\sum_i w_{N,i}^1(x_0) \hat{Q}_1(f_0(z_0);y_1^i,x_1^i) - \E[Q_1(f_0(z_0);Y_1,X_1)| X_0=x_0]\right| > \epsilon\right) \\
&\le T\left(\frac{4T\rho DL}{\epsilon}\right)^{p}\Bigg[\left(\dfrac{16T\sqrt{d}\rho M}{\epsilon}\right)^{d}\exp\left(-\frac{2}{N}\left(Ng\left(\frac{\epsilon}{16TM}\right)^{d}+1-CN^\delta\right)^2\right) \\
&\;\;\;\;+ 2\left(\dfrac{25}{d}\right)^{d}N^{2d}\exp\left(-\frac{CN^\delta\epsilon^2}{64T^2\sigma^2}\right)\Bigg],
\end{align*}
for all $\epsilon \ge 8TM\left(\dfrac{CN^\delta-1}{Ng}\right)^{1/d}$ and $N\ge 2d$. From this we deduce that for any $\alpha \in (0,1)$,
\begin{align*}
&P\bigg(2\sup_{z_0\in\mathcal{Z}} \bigg|\sum_i w_{N,i}^1(x_0) \hat{Q}_1(f_0(z_0);y_1^i,x_1^i) - \E[Q_1(f_0(z_0);Y_1,X_1)| X_0=x_0]\bigg| > \epsilon\bigg)\\
&\;\;\;\; \le \alpha
\end{align*}
is implied by the following system of inequalities:
\begin{align*}
&\frac{CN^\delta\epsilon^2}{64T^2\sigma^2} \ge \log\frac{1}{\alpha} + 2d\log N + p\log\frac{1}{\epsilon} + \log\left(4T\left(\frac{25}{d}\right)^d(4T\rho DL)^p\right),\\
&2\left(\sqrt{N}g\left(\frac{\epsilon}{16TM}\right)^d-CN^{\delta-1/2}\right)^2 \ge \log\frac{1}{\alpha} + (p+d)\log\frac{1}{\epsilon} \\
&\;\;\;\;+ \log(2T(4T\rho DL)^p(16T\sqrt{d}\rho M)^d),\\
&\epsilon \ge 8TM\left(\dfrac{CN^{\delta-1}}{g}\right)^{1/d}.
\end{align*}
Following some algebraic manipulations, we see the above system of inequalities is implied by:
\begin{align*}
&\epsilon \ge \frac{TC_1}{N^{\delta/2}}\left(\sqrt{\log\frac{1}{\alpha}} + \sqrt{2d\log N} + C_2\right) \\
&\;\;\;\;+ TC_3\left(\left(\frac{1}{N}\log\frac{1}{\alpha}\right)^{1/2d} +C_4N^{(\delta-1)/d} + \frac{C_5}{N^{1/2d}}\right),
\end{align*}
where
\begin{align}\label{eq:finiteconstants}
&C_1 = \frac{8\sigma}{\sqrt{C}}, \\
&C_2 =  \sqrt{\log\left(4T\left(\frac{25}{d}\right)^d(4T\rho DL)^p\right)} + \sqrt{\frac{p\delta}{2}\log \frac{2}{TC_1\sqrt{\log\frac{1}{\alpha}}}},\nonumber\\
&C_3 =  \frac{16M}{(\sqrt{2} g)^{1/d}},\nonumber\\
&C_4 = \left(\sqrt{2}C\right)^{1/d} + \frac{8M}{C_3}\left(\frac{C}{g}\right)^{1/d}, \nonumber\\
&C_5 = \left( \log(2T(4T\rho DL)^p(16T\sqrt{d}\rho M)^d)\right)^{1/2d} + \left((p+d)\delta\log \frac{2}{TC_1\sqrt{\log\frac{1}{\alpha}}}\right)^{1/2d}.\nonumber
\end{align}
\end{proof}

\section{Computational Examples}\label{sec:computation}
In this section, we illustrate the practical applicability of our approach with two examples using synthetic data. These examples also serve to demonstrate the value of accounting for auxiliary data.

\subsection{Multistage Inventory Control}
First, we consider a multistage inventory control problem \citep{bertsimas2015}, in which we manage the inventory level of a single product subject to $T$ periods of uncertain demand. At each time step, we observe auxiliary data, which may include data about the product as well as data on time-varying external factors that may be used to predict demand, such as the season of the year, the price of the S\&P 500 index, or the demand for a similar product during the previous time period. We also have historical data of the demand for products we've sold in the past as well as the corresponding auxiliary data for each of these products.

At the beginning of time period $t$, we observe the demand $y_t$ and the new auxiliary covariates $x_t$. Demand can be served by ordering $z^2_t$ units at price $c_2$ for immediate delivery or by ordering $z^1_t$ units at price $c_1 < c_2$ for delivery at the beginning of the next time period. If there is a shortfall in the inventory, orders can be backlogged, incurring a cost of $-c_b$ per unit. If there is excess inventory at the end of a time period, we pay a holding cost of $c_h$ per unit. In addition, there is an ordering budget, so the cumulative advance orders ($\sum_{s \le t} z^1_s$) must not exceed
$\bar{z}_{\text{tot},t}$ at any time $t$. We assume all ordering and inventory quantities are continuous, and we represent the amount of inventory at the end of time period $t$ by $I_t$ (with $I_{-1} = 0$ and $z^1_{-1} = 0$). If demand is known, we have the following deterministic formulation of the problem.

\begin{equation*}
\begin{aligned}
& \min & & \sum_{t=0}^T c_1 z_t^1 + c_2 z_t^2 + \max\{c_b I_t, c_h I_t\}  \\
& \text{s.t.} & & I_{t+1} = I_t + z_t^1 + z_{t+1}^2 - y_{t+1} \;\;\;\;\; \forall t = -1,\ldots,T-1 \\
&&& \sum_{s=0}^t z^1_s \le \bar{z}_{\text{tot},t} \;\;\;\;\; \forall t = 0, \ldots, T \\
&&& z_t^1,z_t^2 \ge 0 \;\;\;\;\; \forall t = 0,\ldots,T.
\end{aligned}
\end{equation*}

We used the parameters $c_1 = 5$, $c_2 = 10$, $c_h = 5$, and $c_b = -10$. We assumed the initial inventory to be 0 and set $\bar{z}_{\text{tot},t} = 50(t+1)$. To generate training data, we sampled $x^i_t$ independently from a 3 dimensional AR(1) process such that $x^i_t = 0.7 x_{t-1}^i + w^i_t$, where $w^i_t$ is a sample of a $\mathcal{N}(0,I_3)$ random variable. We used a factor model for the demand.
\begin{equation*}
y^i_t = \max\left\{0, 50 + 12 a^T_t(x^i_{t-1} + 0.25\phi_t) + 5b^T_tx^i_{t-1} \theta_t \right\} \;\;\;\;\forall t = 1,\ldots,T
\end{equation*}
where $\{\phi_t\}$ are drawn independently from a 3 dimensional standard Gaussian and $\{\theta_t\}$ are drawn independently from a 1 dimensional standard Gaussian. At each time step, the factor loadings, $a_t$ and $b_t$, are permutations of $\begin{pmatrix} 0.8 & 1 & 1\end{pmatrix}^T$ and $\begin{pmatrix} -1 & 1 & 0\end{pmatrix}^T$, respectively (held constant for all samples). The results we present here show the average cost of policies based on out-of-sample testing as a function of the amount of training observations, $N$. All results are averaged over one hundred realizations of training sets.

\begin{figure}
\centering
\includegraphics[width=0.8\linewidth]{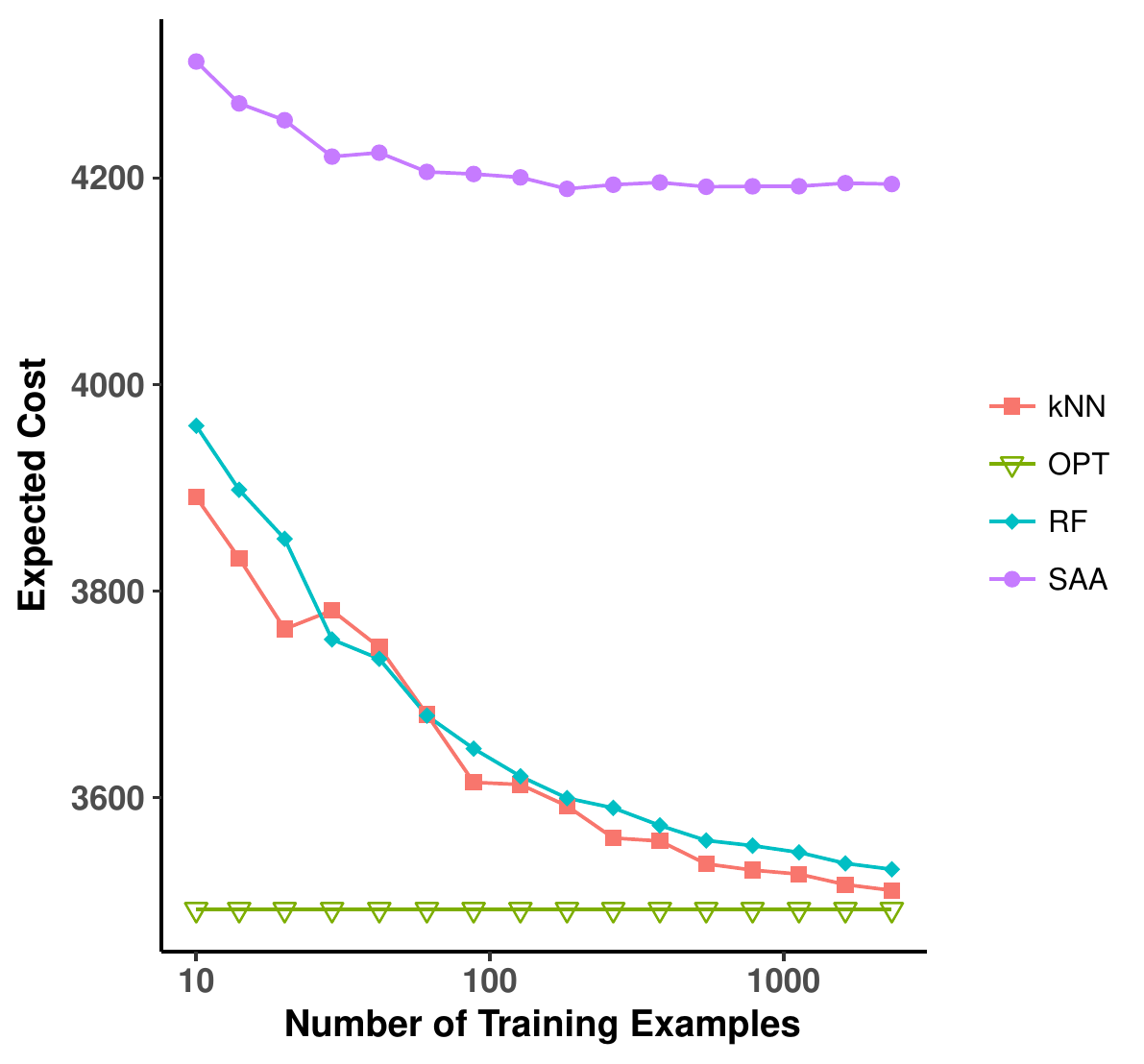}
\caption{Out of sample results with various weight functions for a twelve stage inventory control problem. Vertical axis represents expected cost of policy (smaller is better).}\label{fig:mic}
\end{figure}

Figure \ref{fig:mic} shows the expected cost of policies learned using our method versus the number of training samples. We see that the SAA approach, which ignores the auxiliary data, is suboptimal. Our method, using the $k$-nearest neighbor and random forest weight functions, is asymptotically optimal. We obtain a reduction in cost of over 15\% by accounting for auxiliary data.

\subsection{Multistage Lot Sizing}\label{sec:mls}
For our second computational example, we consider a multistage lot sizing problem \citep{bertsimas2015}. This problem is similar to the multistage inventory control problem, but it includes binary decision variables. The continuous ordering decision for immediate delivery, $z_t^2$, is replaced with $M$ binary ordering decisions, $z_{tj}^2$, $j=1,\ldots,M$. Each of these decisions corresponds to a quantity, $q_j$, which is delivered immediately for cost $c_{2j} > c_1$ per unit. Additionally, there is no longer the option to backorder demand. All demand must be satisfied immediately. These restrictions make the problem more realistic because it is often not feasible to produce an arbitrary amount of a product immediately, and it is difficult to estimate the cost of lost customer goodwill due to backordering. Instead, in order to meet demand, the decision maker must buy from another supplier a fixed quantity of product at a higher price.

If demand is known, we have the following deterministic formulation (where we assume $I_{-1} = 0, z^1_{-1} = 0$, and $y_0=0$).
\begin{equation*}
\begin{aligned}
&\min && \sum_{t=0}^T c_1 z_t^1 + \sum_{j=1}^M c_{2j}q_jz_{tj}^2 + c_hI_t \\
& \text{s.t.} & & I_{t+1} = I_t + z_t^1 + \sum_{j=1}^M c_{2j} q_j z_{tj}^2 + c_h I_t \\
&&& \sum_{s=0}^t z^1_s \le \bar{z}_{\text{tot},t} \;\;\;\;\; \forall t = 0,\ldots,T\\
&&& z_t^1,I_t \ge 0 \;\;\;\;\; \forall t = 0,\ldots,T \\
&&& z_{tj}^2 \in \{0,1\} \;\;\;\;\; \forall t=0,\ldots,T\;\;\forall j=1,\ldots,M
\end{aligned}
\end{equation*}

\begin{figure}
\centering
\includegraphics[width=0.8\linewidth]{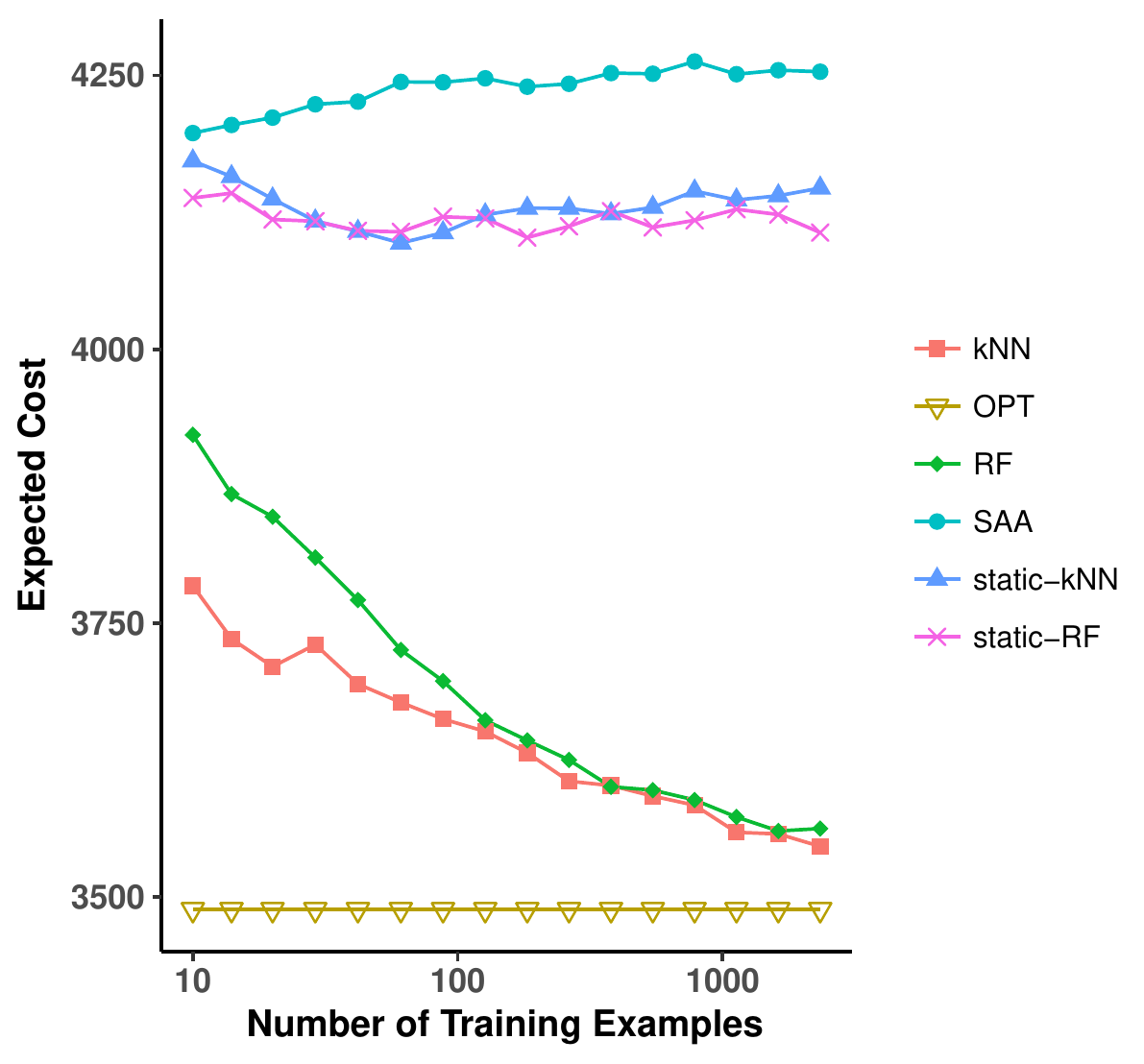}
\caption{Average out-of-sample cost of policies computed using various weight functions for twelve stage lot sizing problem. Horizontal axis shows number of training examples.}\label{fig:mls}
\end{figure}

We used the same parameters and data generating procedure as in the multistage inventory control example. The only differences were that we capped $y_t$ at 200 (to ensure feasibility) and we drew $c_{2j}$ independently from a uniform distribution on $(5,10)$ for each $j$.

To solve the problem, we use an approximate DP algorithm. To develop the algorithm, we recall that basestock policies are optimal for a wide variety of inventory control problems. A basestock policy is one in which there is some ideal amount of inventory, $r_t$, which we desire at the start of time period $t$. If we have less than $r_t$, we place advanced orders (if available) to have this amount. If we have more than $r_t$, we order nothing in advance. We then serve the remaining demand with immediate orders. A basestock policy will not be optimal for the lot sizing problem because of the nonconvexity of the value function, but it does provide a reasonable approximation. To account for auxiliary data, we have $N$ different basestock amounts for each time period. Therefore, $r_t^i$ denotes the target basestock at time $t$ when the observed axuiliary data is $x_t^i$. Parametrizing the policy space with $NT$ parameters greatly reduces the amount of computation required to solve the problem and allows us to solve much larger problem instances than we could otherwise.

Figure \ref{fig:mls} shows the results of our method for the SAA, $k$-nearest neighbor, and random forest weight functions on a twelve stage lot sizing problem. The SAA method is again clearly suboptimal. We see that the static $k$NN and static RF methods, which only use the auxiliary covariates at time 0, offer a significant improvement over the SAA method. However, the $k$NN and RF methods which take into account the auxiliary data that arrives over time outperforms all of the above methods, again illustrating the value of auxiliary data. With very little additional computational cost, we are able to obtain an improvement in cost of nearly 15\%.

\section{Conclusion}
In this paper, we introduced a data-driven framework for solving multistage optimization problems under uncertainty with auxiliary covariates. We demonstrated how to develop specific methods by integrating predictive machine learning methods such as $k$NN, CART, and random forests. Our approach is well suited for multistage optimization problems in which the distribution of the uncertainties is unknown, but samples of the uncertainty and auxiliary data are available.

We demonstrated that our method with the $k$-nearest neighbor, CART, and random forest weight functions are asymptotically optimal. We also provided finite sample guarantees for the method with $k$NN weight functions. Additionally, we showed how to apply the framework with two computational examples. Because we can think of (\ref{eq:empirical}) as a dynamic programming problem, we have at our disposal a variety of exact and approximate solution techniques. The problem is often tractable in practice and can lead to significant improvements over methods that ignore auxiliary data.

We leave for future work the extension in which the decision affects the distribution of the uncertainty. This type of problem appears in applications such as pricing where the choice of price affects the distribution of the demand. We also leave for future research the development of efficient variants of our methods for specific applications. In a world in which the availability of data continues to grow, our proposed approach utilizes this data efficiently and has the potential to make a significant impact in OR applications.


\bibliographystyle{spbasic}      
\bibliography{references}    

\appendix

\section{Additional Results on Random Forest Weight Functions}
Here we provide an additional result on the strong asymptotic optimality of the method with random forest weight functions in the single stage setting. Here, we consider random forests as defined in \citet{wager2015}. The random forest consists of an ensemble of trees, each trained on a subsample of the data of size $s$. Each of the trees in the forest is a regular, random-split, and honest regression tree as in Definition \ref{def:tree}. The prediction of the random forest at $x$ is given by
\begin{equation}\label{eq:rfdef}
R(x,X_1,Y_1,\ldots,X_N,Y_N) = \binom{N}{s}^{-1}\sum_{i_1,\ldots,i_s}\E_\xi[T(x,\xi,X_{i_1},Y_{i_1},\ldots,X_{i_s},Y_{i_s})].
\end{equation}
In practice, this is estimated by training trees on random subsamples of the data and random draws of $\xi$.
\begin{assumption}[Random Forest Specification]\label{as:rfweightfunction}
The random forest, as defined in (\ref{eq:rfdef}), has random-split, regular, and honest regression trees as its base learners. In addition, $\lambda \le 0.2$, and the subsample size, $s_N$, scales with $N^\beta$. That is, $s_N = \min(\lfloor CN^\beta\rfloor,N-1)$ with $C > 0$ and $\beta \in \left[\left(2+\frac{\pi\log(1-\lambda)}{d\log\lambda}\right)^{-1},\frac{1}{2}\right)$.
\end{assumption}
This assumption ensures the forest consists of diverse trees, each with low bias, so they can be aggregated into a consistent regressor.

\begin{theorem}\label{thm:rf2}
Suppose $\mathcal{Z}\subset\R^p$ is nonempty and compact, and the training data is i.i.d. In addition, suppose $|c(z;y)| \le 1$, for all $z\in\mathcal{Z}$, $y\in\mathcal{Y}$, and that $c(z;y)$ is a well-defined $L$-Lipschitz continuous function of $z$ for all $y$. Finally, suppose $X^i$ is uniformly distributed on $[0,1]^d$ and $\E[c(z;Y)|X=x]$ is an $M$-Lipschitz continuous function of $x$ for all $z$. Let $w_{N,i}(x)$ be the random forest weight function, satisfying assumption \ref{as:rfweightfunction}. Then $\{\hat{z}_0^N(x)\}$, a sequence of optimal solutions to
\begin{equation*}
\min_{z\in\mathcal{Z}} \sum_{i=1}^N w_{N,i}(x) c(z;y^i),
\end{equation*}
is strongly asymptotically optimal with respect to the true problem, $\min\limits_{z\in\mathcal{Z}} \E[c(z;Y)|X=x]$.
\end{theorem}

This result shows that our method can be strongly asymptotically optimal for the single stage problem with random forest weight functions. Some of the assumptions are slightly stronger than in Theorem \ref{thm:rf}. For example, we require that the value functions are bounded and the random forests are slightly different. However, we do not require that the minimum number of training samples per leaf, $k$, grows with $N$ as we do for Theorem \ref{thm:rf}. This is consistent with Breiman's original random forest algorithm in which $k$ is fixed at 1. To prove this theorem, we first prove a result on the strong consistency of the random forest estimator.

\begin{lemma}\label{lemma:rfstrongconsistency}
Suppose $(X_1,Y_1),\ldots,(X_N,Y_N)\in\R^{d+1}$ are i.i.d. samples, the distribution of $X_1$ uniform on $[0,1]^d$, and $|Y| \le 1$, a.s. Let $\mu(x) = \E[Y_1|X_1=x]$ be a Lipschitz continuous function. Define $\hat{\mu}_N(x)$ to be the prediction of a random forest that satisfies Assumption \ref{as:rfweightfunction}. Then, almost surely,
\begin{equation*}
\hat{\mu}_N(x) \to \mu(x)
\end{equation*}
for $x$ a.e.
\end{lemma}
\begin{proof}
We first note that the bias of the predictions goes to 0, $|\E\hat{\mu}_N(x) - \mu(x)| \to 0$, for $x$ a.e. by Theorem 3 from \cite{wager2015}. (The assumptions are satisfied by Assumption \ref{as:rfweightfunction}.) Next, we define $h_x(Z_1,\ldots,Z_N) = \hat{\mu}_N(x)$, where $Z_i=(X_i,Y_i)$. We note that for any $Z_k,Z_k'\in\mathcal{X}\times\mathcal{Y}$,
\begin{equation*}
|h_x(Z_1,\ldots,Z_k,\ldots,Z_N) - h_x(Z_1,\ldots,Z'_k,\ldots,Z_N)| \le \frac{2s_N}{N}.
\end{equation*}
This is due to the assumption that $|Y_i| \le 1$. Since each tree predicts at $x$ by averaging the $Y_i$s corresponding to training samples in the same partition of the feature space as $x$, the prediction of any tree is bounded between -1 and 1. Changing a single training sample only affects the trees trained on subsets of the data including that example. This only affects a fraction $s_N/N$ of the trees in the forest. Since the prediction of the random forest is the average of the predictions of all the trees, the most the prediction can be changed by altering a single training sample is $2s_N/N$. Applying McDiarmid's inequality, we have, for $\epsilon > 0$,
\begin{equation*}
P(|\hat{\mu}_N(x)-\E\hat{\mu}_N(x)| > \epsilon) \le 2\exp\left(-\frac{N\epsilon^2}{2s_N^2}\right).
\end{equation*}
By our assumption, this can be rewritten
\begin{equation*}
P(|\hat{\mu}_N(x)-\E\hat{\mu}_N(x)| > \epsilon) \le 2\exp\left(-\frac{N^{1-2\beta}\epsilon^2}{C^22}\right),
\end{equation*}
where $\beta < 1/2$. From this we see $\sum_{N=1}^\infty P(|\hat{\mu}_N(x)-\E\hat{\mu}_N(x)| > \epsilon) < \infty$, so, by the Borel Cantelli lemma, $|\hat{\mu}_N(x)-\E\hat{\mu}_N(x)|\to 0$ a.s. Combining the two results with the triangle inequality completes the proof.
\end{proof}

\begin{proof}[Theorem \ref{thm:rf2}]
We need to show
\begin{equation*}
\sup_{z\in\mathcal{Z}}\left|\sum_{i=1}^N w_{N,i}(x) c(z;y^i) - \E[c(z;Y)|X=x]\right| \to 0
\end{equation*}
a.s. for $x$ a.e. The desired result then follows from lemmas \ref{lemma:uniformization} and \ref{lemma:opt}. If we ignore the supremum, we can apply lemma \ref{lemma:rfstrongconsistency} to see it goes to 0 a.s. Next, we apply lemma \ref{lemma:final} to see that the convergence holds simultaneously for all $z$ with probability 1. Finally, we apply lemma \ref{lemma:uniform} to show the convergence is uniform over $z$ a.s.
\end{proof}

\section{Decomposition Algorithm}\label{app:decomp}
For the reader's convenience, we present a decomposition algorithm, similar to that of \citet{shapiro2011}, tailored for use with our methods. Given weight functions, which we compute from an appropriate machine learning algorithm, our goal is to solve (\ref{eq:empirical}). If the decisions spaces and state spaces are finite sets with small cardinality, it may be possible to solve the problem exactly using the classical dynamic programming algorithm (see, for example, \citet{bertsekas2017}). However, in many OR problems, this is not the case, and we need to deal with continuous decision and state spaces.

We note that it is possible to formulate (\ref{eq:empirical}) as a large, but finite sized, single stage optimization problem. To do so, we create $N$ copies of $z_1$, $N^2$ copies of $z_2$, etc. These copies of the decision variables represent our contingency plan. For each potential realization of $Y_1,\ldots,Y_t$, we have a distinct copy of $z_t$. If the $f_t$ and $g_t$ functions are linear and the $Z_t$ sets are polyhedral, the resulting problem will be a linear optimization problem, which can be solved in time that is polynomial in the size of the formulation. However, the number of variables in the formulation is $O(N^T)$, and this formulation becomes impractical for moderately sized $N$ and $T$. In order to solve larger problems, we resort to a Benders-like decomposition approach. (For a review of Benders decomposition methods, see, for example, \citet{murphy2013}.)

Algorithm \ref{alg:sddp} describes the approach. The main idea is that we maintain a piecewise linear, convex lower bound, $\psi_t(s_t,x_{t-1})$, on $\sum\limits_{i=1}^N w_{N,i}^t(x_{t-1}) \hat{Q}_t(s_t;y_t^i,x_t^i)$ for each $t=1,\ldots,T$. We have the relaxed problems:
\begin{equation}\label{eq:relaxed}
P_t(s_t,x_t,y_t,\psi_{t+1}) := \min_{z_t\in Z_t(s_t,y_t)} g_t(z_t) + \psi_{t+1}(f_t(z_t),x_t)
\end{equation}
(For $t=0$, there is no dependence on $y_0$.) If $g_t$ and $f_t$ are linear functions and $Z_t$ is a polyhedral set, then the relaxed problem can be reformulated as a linear optimization problem because $\psi_{t+1}$ is the maximum of a finite set of linear functions.

The algorithm consists of two main steps, which are repeated until convergence. First, $M$ sample paths for $X$ and $Y$ are sampled (with replacement) from the training data. Then, in the forward step, for each of these sample paths, trial states are computed by solving the relaxed problems from $t=0$ to $t=T$, assuming the system evolves according to the corresponding sample path. Using the costs of each of these sample paths, we can compute a statistical upper bound on the optimal value of (\ref{eq:empirical}) (assuming $M$ and $N$ are such that a central limit theorem is a reasonable approximation).

In the backward step, we update the $\psi_{t}$ functions. We proceed backwards, starting with $t=T$, and solve $P_t$ for each of the trial states. We compute the cut coefficients (which we will discuss in more detail next), and then average them across all possible realizations $y_t^j$ and $x_t^j$, according to the distribution $\{w_{N,i}^t(x_{t-1}^j)\}_{i=1}^N$. We then update $\psi_t$ with this new cut. Finally, we update the lower bound on the optimal value of the problem by solving $P_0(s_0,x_0,\psi_1)$.

There are several possible stopping criteria we can use. One is to stop when the statistical upper bound (line \ref{line:ub} in Algorithm \ref{alg:sddp}) is within a specified $\epsilon > 0$ of the lower bound (line \ref{line:lb}). This will give us an $\epsilon$-optimal solution with probability at least $1-\alpha$ (assuming $M$ is sufficiently large and $N$ is much larger than $M$ so that the central limit theorem is a reasonable approximation). Alternatively, we can stop when the lower bound stabilizes or after a fixed number of iterations. All three of these can allow the algorithm to construct lower bounds, $\{\psi_t\}$, that are reasonable approximations to the value functions.

\begin{algorithm}
\SetKwInOut{Input}{Input}
\caption{SDDP algorithm for multistage optimization.}\label{alg:sddp}
\Input{$s_0$, $x_0$, weight functions $\{w_{N,i}^t(x_{t-1})\}_{i=1,\ldots,N;t=1,\ldots,T}$}
Initialize: $LB \leftarrow -\infty$, $UB \leftarrow \infty$,  and initial lower bounds $\{\psi_t\}_{t=1}^T$

\While{stopping criterion not satisfied}{
Sample $M$ scenarios $\{y^j_1,x^j_1,y^j_2,\ldots,x^j_{T-1},y^j_T\}_{j=1}^M$ with replacement from the training set, such that $y^j_1$ and $x^j_1$ are sampled from the probability distribution defined by $w^1_{N,i}(x_0)$, $y^j_2$ and $x^j_2$ are sampled from the probability distribution defined by $w^2_{N,i}(x_1^j)$, etc.

\bigskip
/* Forward step */

\For{$j = 1,\ldots,M$}{
Initialize $s_0^j \leftarrow s_0$

\For{$t = 0,\ldots,T$}{
Solve problem $P_t(s_t^j,x_t^j,y_t^j,\psi_{t+1})$ for optimal solution $z_t^j$

Update state $s_{t+1}^j \leftarrow f_t(z_t^j)$
}
$v^j \leftarrow \sum\limits_{t=0}^T g_t(z_t^j)$
}
\bigskip
/* Statistical upper bound */

$\hat{\mu} \leftarrow \frac{1}{M}\sum\limits_{j=1}^M v^j$ and $\hat{\sigma}^2 \leftarrow \frac{1}{M-1}\sum\limits_{j=1}^M (v^j - \hat{\mu})^2$

$UB \leftarrow \hat{\mu} + z_{\alpha/2}\frac{\hat{\sigma}}{\sqrt{M}}$ \label{line:ub}

\bigskip
/* Backward Step */

\For{$t = T,\ldots,1$}{
\For{$j = 1,\ldots,M$}{
\For{$i = 1,\ldots,N$}{
Solve $P_t(s_t^j,x_t^i,y_t^i,\psi_{t+1})$ to compute cut coefficients $(\beta_t^{ji},\pi_t^{ji})$
}

Update $\psi_{t}(s_{t},x_{t-1}^j) \leftarrow \max\left\{\psi_{t}(s_{t},x_{t-1}^j), \sum\limits_{i=1}^N w^t_{N,i}(x_{t-1}^j)(\beta_t^{ji} + (\pi_t^{ji})^T s_t) \right\}$
}
}
\bigskip
/* Lower bound update */

Solve $P_0(s_0,x_0,\psi_1)$ and set $LB$ to its optimal value\label{line:lb}

}
\end{algorithm}

The crucial component required for this algorithm to work is the cuts. We begin with a definition from \cite{zou2016}.

\begin{definition}[Valid, tight, and finite cut]
Let $(\beta_t,\pi_t)$ be the stage $t$ cut coefficients computed in the backward step of Algorithm \ref{alg:sddp}. We say that the cut is:
\begin{enumerate}
\item Valid if
\begin{equation*}
\sum_{i=1}^N w^t_{N,i}(\hat{x}_{t-1}) \hat{Q_t}(s_t;y_t^i,x_t^i) \ge \beta_t + \pi_t^T s_t \;\;\;\;\;\; \forall s_t.
\end{equation*}
\item Tight if
\begin{equation*}
\sum_{i=1}^N w^t_{N,i}(\hat{x}_{t-1}) P^*_t(\hat{s}_t,x_t^i,y_t^i,\psi_{t+1}) = \beta_t + \pi_t^T \hat{s}_t,
\end{equation*}
where $P_t^*$ represents the optimal value of (\ref{eq:relaxed}), $\hat{s}_t$ is the trial state computed during the corresponding forward pass of the algorithm, and $\hat{x}_{t-1}$ is the auxiliary covariate from the forward pass of the algorithm.
\item Finite if solving (\ref{eq:relaxed}) for fixed $\psi_{t+1}$ can only generate finitely many possible cuts.
\end{enumerate}
\end{definition}

Under the conditions that $f_t$ and $g_t$ are linear functions, $Z_t$ are polyhedral sets, and at every stage, (\ref{eq:relaxed}) is feasible with finite optimal value, it is shown in \cite{shapiro2011} that the SDDP algorithm will converge to an optimal solution in a finite number of iterations with probability 1, provided the cuts used are valid, tight, and finite. \cite{zou2016} showed that this result also holds if the state variables are purely binary, instead of continuous.

The validity of the cuts ensures that the $\{\psi_t\}$ functions maintain lower bounds on the value functions at each stage. Next, we describe several classes of valid cuts that can be used within the SDDP algorithm.

\subsubsection*{Benders' Cut}
A well known class of cuts is the Bender's cut \citep{benders1962}. These cuts are valid for linear problems, even with integer constraints, and are tight for linear optimization problems (or more generally convex optimization problems) in which strong duality holds. To compute the cuts for stage $t$ in the SDDP algorithm we solve the following form of $P_t(s_t^j,x_t^i,y_t^i,\psi_{t+1})$:
\begin{equation*}
\begin{aligned}
& \min_{z_t,s_t} & & g_t(z_t) + \psi_{t+1}(f_t(z_t),x_t^i) &\\
& \text{s.t.} & & z_t \in Z_t(s_t,y^i_t) &\\
&&& s_t = s_t^j.
\end{aligned}
\end{equation*}

 We then let $\pi_t^{ji}$ be the optimal dual solution (of the LO relaxation if there are integer variables) corresponding to the indicated constraint and set $\beta_t^{ji} = P^*_t(s_t^j,x_t^i,y_t^i,\psi_{t+1}) - (\pi_t^{ji})^T s_t^j$, where $P^*_t(s_t^j,x_t^i,y_t^i,\psi_{t+1})$ is the optimal value of the above problem. In order for these cuts to be finite, we should always use basic solutions for $\pi_t^{ji}$.
 
 \subsubsection*{Integer Optimality Cut}
 If the state space is binary, the SDDP algorithm with Benders' cuts is not guaranteed to produce an optimal solution. This is because they are not guaranteed to be tight. Instead, we can solve the above integer optimization problem to optimality and choose cuts defined by the linear expression:
 \begin{equation*}
 (P_t^*(s_t^j,x_t^i,y_t^i,\psi_{t+1}) - L_t)\left(\sum_k (s_{tk}(1-s^j_{tk}) + (1-s_{tk})s^j_{tk})\right) + P_t^*(s_t^j,x_t^i,y_t^i,\psi_{t+1}).
 \end{equation*}
These cuts are valid, tight, and finite when the state space is binary. However, they tend to be very ineffective in practice.

\subsubsection*{Lagrangian Cut}
The third class of cut we describe was introduced by \cite{zou2016} and shown to be valid and tight when the state space is binary. They are much more effective than the integer optimality cuts in practice. These cuts are computed by solving the Lagrangian dual problem:
\begin{equation*}
\max_{\pi_t^{ji}} \mathcal{L}_t(\pi_t^{ji}) + (\pi_t^{ji})^Ts_t^j,
\end{equation*}
where
\begin{equation*}
\begin{aligned}
\mathcal{L}_t(\pi_t^{ji}) =& \min_{z_t,s_t} & & g_t(z_t) + \psi_{t+1}(f_t(z_t),x_t^i) - (\pi_t^{ji})^T s_t \\
& \text{s.t.} & & z_t \in Z_t(s_t,y^i_t) \\
&&& s_t \in [0,1]^p.
\end{aligned}
\end{equation*}
We then use the cut with $\pi_t^{ji}$ equal to the optimal solution of the Lagrangian dual problem and $\beta_t^{ji}$ equal to the optimal value of $\mathcal{L}_t(\pi_t^{ji})$.

These three classes of cuts allow us to solve problems where the state space is continuous or pure binary with the SDDP algorithm. When the state space is a mixed integer set, we can perform a binary expansion to desired accuracy on the continuous variable to convert the problem to the pure binary case. Of course, we can also combine different classes of cuts, and this can speed convergence.

\end{document}